\documentclass[twocolumn,table,dvipsnames]{ncu}   

\usepackage{silence}
\usepackage{amsthm}
\usepackage{nicefrac}
\usepackage{mathrsfs}
\usepackage{dsfont}
\usepackage{upgreek}
\usepackage{gensymb}

\usepackage{algorithm}
\usepackage{algpseudocode}
\usepackage{listings}

\usepackage{array}
\usepackage{makecell}
\usepackage{adjustbox}
\usepackage{tabularx}
\usepackage{diagbox}
\usepackage{longtable}
\usepackage{siunitx}
\usepackage{wrapfig}
\usepackage{float}
\usepackage{enumitem}

\usepackage{pifont}
\usepackage{fontawesome5}
\usepackage{bbding}
\usepackage{xspace}
\usepackage{textcomp}
\usepackage{epigraph}

\newlength\savewidth

\newcolumntype{x}[1]{>{\centering\arraybackslash}p{#1pt}}
\newcolumntype{y}[1]{>{\raggedright\arraybackslash}p{#1pt}}
\newcolumntype{z}[1]{>{\raggedleft\arraybackslash}p{#1pt}}

\theoremstyle{plain}
\newtheorem{theorem}{Theorem}
\newtheorem{proposition}[theorem]{Proposition}

\theoremstyle{definition}
\newtheorem{definition}[theorem]{Definition}

\theoremstyle{remark}

\usepackage{newfloat}
\usepackage{xcolor}
\usepackage{upquote}
\DeclareFloatingEnvironment[name=Code, placement=tbp, fileext=lol]{codeblock}

\definecolor{pygkeyword}{HTML}{007020}
\definecolor{pygstring}{HTML}{4070A0}
\definecolor{pygcomment}{HTML}{60A0B0}
\definecolor{pygname}{HTML}{06287E}
\definecolor{pyglineno}{RGB}{128,128,128}

\lstdefinestyle{pycode}{
    language=Python,
    basicstyle=\ttfamily\scriptsize,
    keywordstyle=\color{pygkeyword}\bfseries,
    stringstyle=\color{pygstring},
    commentstyle=\color{pygcomment}\itshape,
    emph={torch,math,F,self}, emphstyle=\color{pygname},
    numbers=left,
    numberstyle=\ttfamily\scriptsize\color{pyglineno},
    numbersep=4pt,
    xleftmargin=10pt,
    breaklines=true, breakatwhitespace=false,
    postbreak=\mbox{\textcolor{pyglineno}{$\hookrightarrow$}\space},
    showstringspaces=false, upquote=true,
    columns=fullflexible, keepspaces=true,
    tabsize=4,
}

\definecolor{goodgreen}{HTML}{009000}
\definecolor{badred}{HTML}{ea4335}

\newcommand{\eu}{\mathrm{e}\mkern1mu}

\newcommand{\diff}{\mathop{}\!\mathrm{d}}

\title{Sphere Retraction Normalizations}

\author[1,*]{Jie Zhang}
\author[2,*]{Cheng-Fang Su}
\author[1,3]{Yi-Jui Huang}
\author[1]{Min-Te Sun}

\affiliation[1]{Department of Computer Science and Information Engineering, National Central University, Taiwan}
\affiliation[2]{Department of Applied Mathematics, National Yang Ming Chiao Tung University, Taiwan}
\affiliation[3]{Department of Mathematics, National Central University, Taiwan}
\contribution[*]{Equal contribution}

\abstract{
Residual connections are the de facto mechanism for training deep neural networks stably. Geodesic Normalization (GeoNorm) recasts them on a Riemannian manifold, orthogonalizing each layer output against the current hidden state and applying the resulting update through the Riemannian exponential map. Every hidden state thus keeps a constant $\ell_{2}$-norm, confining the residual stream to a hypersphere. The exponential map, however, is only one member of a broad family of retraction maps. We show that on the hypersphere this entire family collapses to a single scalar design choice. What distinguishes one retraction from another is only how the magnitude of an update is converted into a rotation angle within the plane spanned by the hidden state and the update. This view places Euclidean residual connections and GeoNorm in one framework. Instantiating it with the metric projection retraction and the Cayley retraction yields \textit{Proj-SpheretNorm} and \textit{Cay-SpheretNorm}, which are exactly norm-preserving yet require only algebraic operations. Both prove to be members of a one-parameter family of angular retractions, \textit{$p$-SpheretNorm}, whose rotation angle saturates rather than growing without bound. The two methods above are recovered exactly at $p = 1$ and $p = 2$, while the identity map and GeoNorm arise only as limits at either end. On nanoGPT, all three methods outperform existing lightweight deep connection schemes, and the best validation loss is attained at finite $p$, indicating that the exponential map is not the preferred retraction for spherical residual streams but merely one end of a spectrum.
}

\ncudata[Code]{\codeurl{https://github.com/hazdzz/deep_connection}}
\ncudata[Correspondence]{Jie Zhang (\email{hazdzz@g.ncu.edu.tw}), \\
\hphantom{\textbf{Correspondence:} }Cheng-Fang Su (\email{scf1204@nycu.edu.tw})}
\date{\today}
\preprint{Preprint}

\begin{document}

\tcbset{
    overlay={
        \node[
            anchor=south east,
            at=(frame.south east),
            xshift=-0.15cm,
            yshift=0.25cm,
        ]{
            \raisebox{0.05cm}{%
                \includegraphics[height=2.6cm]{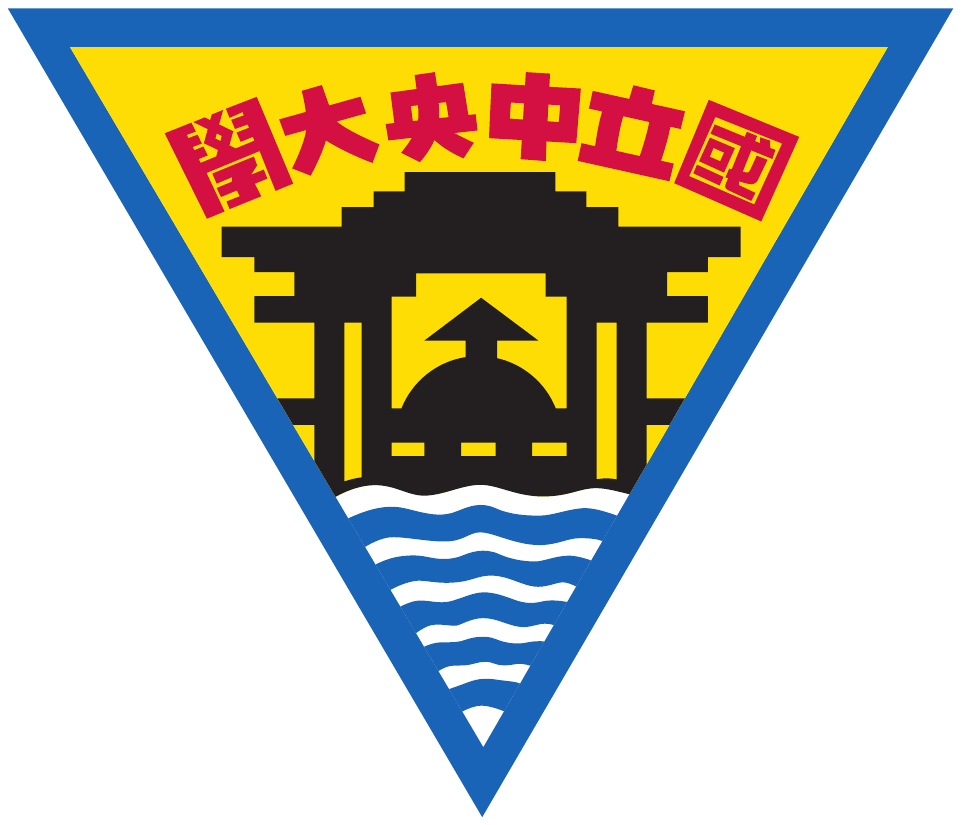}%
            }%
            \hspace{-0.62cm}%
            \includegraphics[height=2.4cm]{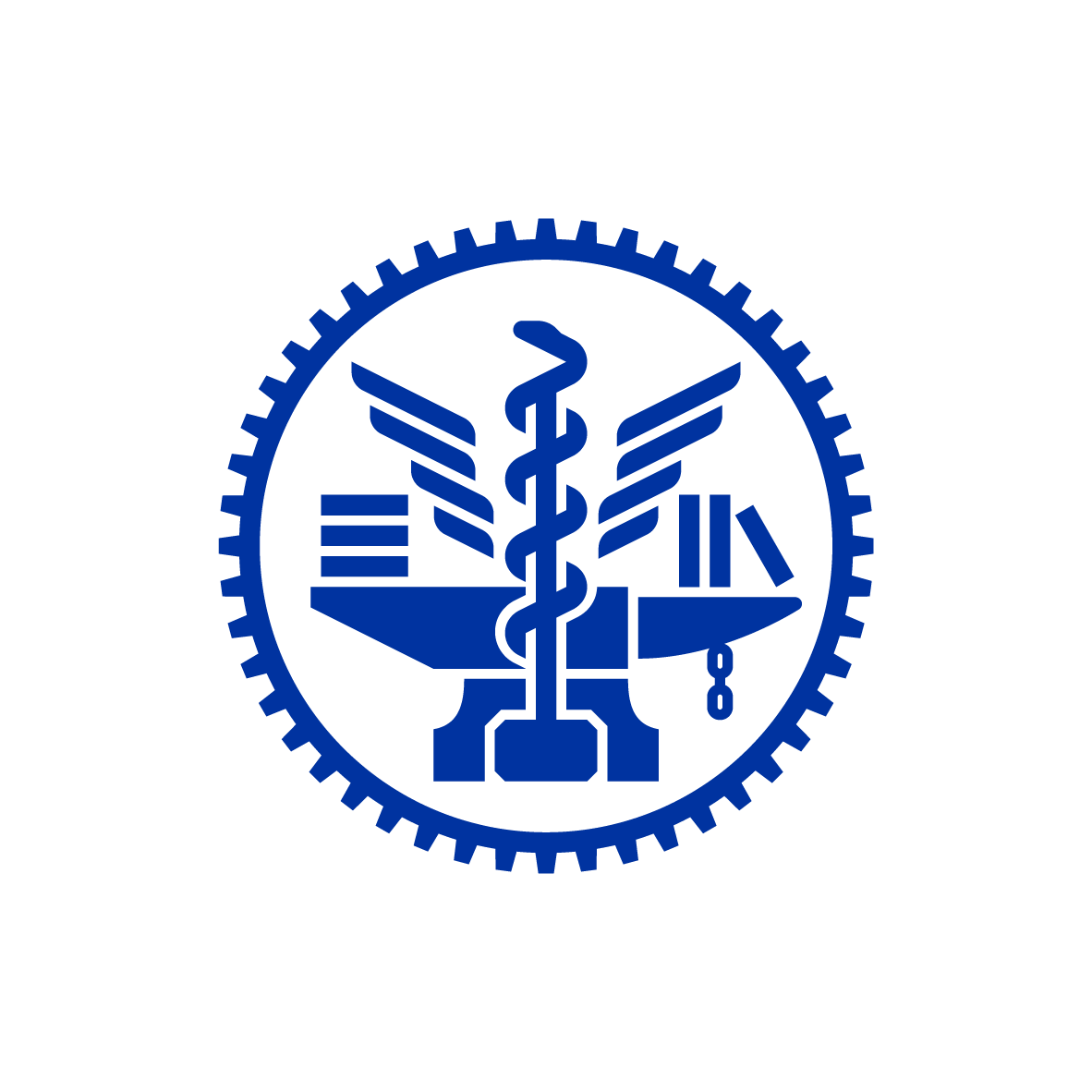}%
        };
    },
}

\maketitle

\epigraph{``Geometry is not true, it is advantageous.''}{--- Henri Poincar{\'{e}}}
\vspace{-2ex}

\section{Introduction}
Training neural networks in a deep and stable manner remains an open problem. Prior research~\citep{li2025mixln,pmlr-v267-kim25u,chen2026post0layernorm} primarily tackles this challenge by combining skip connections~\citep{he2016deep,10.1007/978-3-319-46493-0_38}, normalization methods~\citep{ba2016layer,NEURIPS2019_1e8a1942}, and gating mechanisms~\citep{10.1162/neco.1997.9.8.1735,cho-etal-2014-learning}. We refer to this class of methodologies collectively as \emph{deep connections}. In the Transformer architecture~\citep{NIPS2017_3f5ee243}, Post-LayerNorm
(Post-LN)~\citep{NIPS2017_3f5ee243} and Pre-LayerNorm
(Pre-LN)~\citep{wang-etal-2019-learning-deep,pmlr-v119-xiong20b} have emerged as the two most common choices for the placement of deep connections.

However, both Pre-LN and Post-LN exhibit inherent limitations. Transformers with Pre-LN suffer from the curse of depth~\citep{sun2025curse}, where the contribution of deeper layers progressively diminishes. Moreover, the unbounded residual stream in Pre-LN allows massive activations to accumulate across layers, which has been linked to more pronounced attention sinks --- a phenomenon in which attention scores concentrate excessively on the first token~\citep{xiao2024efficient,sun2024massive,gu2025when}. By contrast, Transformers with Post-LN, on the other hand, suffer from large gradients near the output layer at initialization~\citep{pmlr-v119-xiong20b}, which destabilizes training and necessitates a carefully tuned learning-rate warm-up schedule that slows down optimization and adds hyperparameter tuning.

\begin{figure}[!ht]
\centering
\begin{subfigure}[t]{0.4\textwidth}
    \centering
    \includegraphics[width=0.4\linewidth]{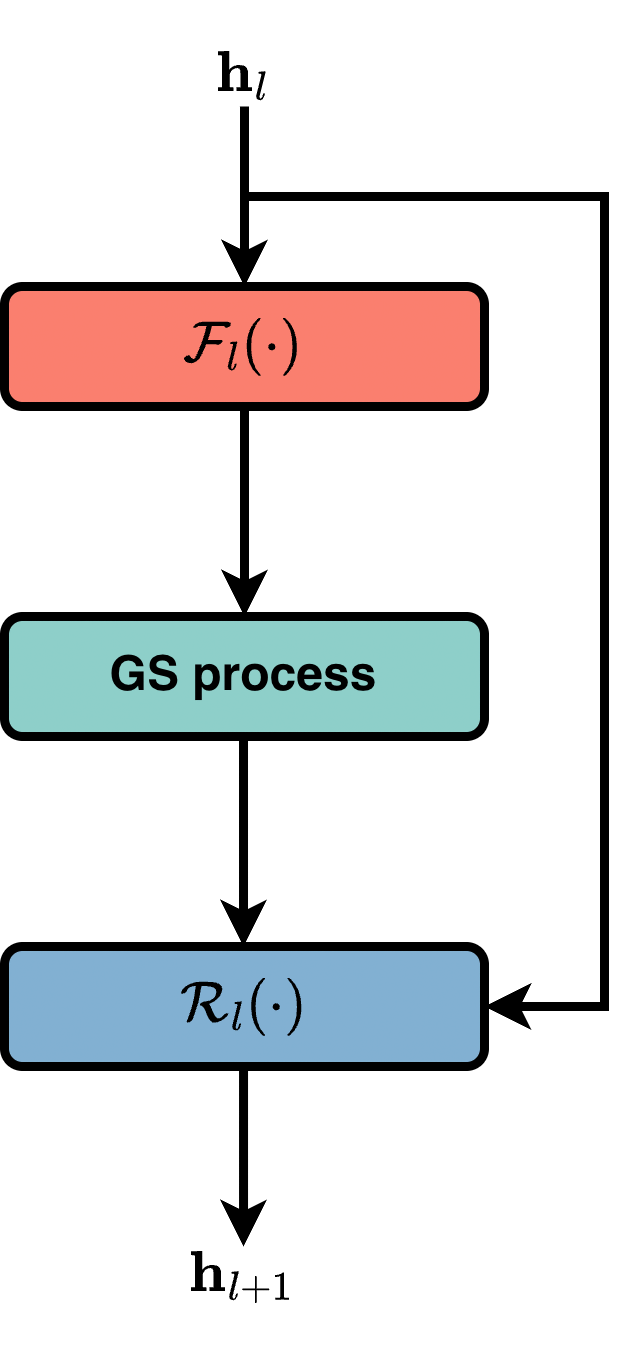}
    \caption{The whole process of SpheretNorm. $\mathcal{F}_{l}(\cdot)$ represents a (non-)linear mapping, the GS process stands for the Gram–Schmidt process, and $\mathcal{R}_{l}(\cdot)$ denotes the retraction map.}
    \label{fig:spheretnorm_process}
\end{subfigure}
\vspace{1em}
\begin{subfigure}[t]{0.4\textwidth}
    \centering
    \includegraphics[width=0.4\linewidth]{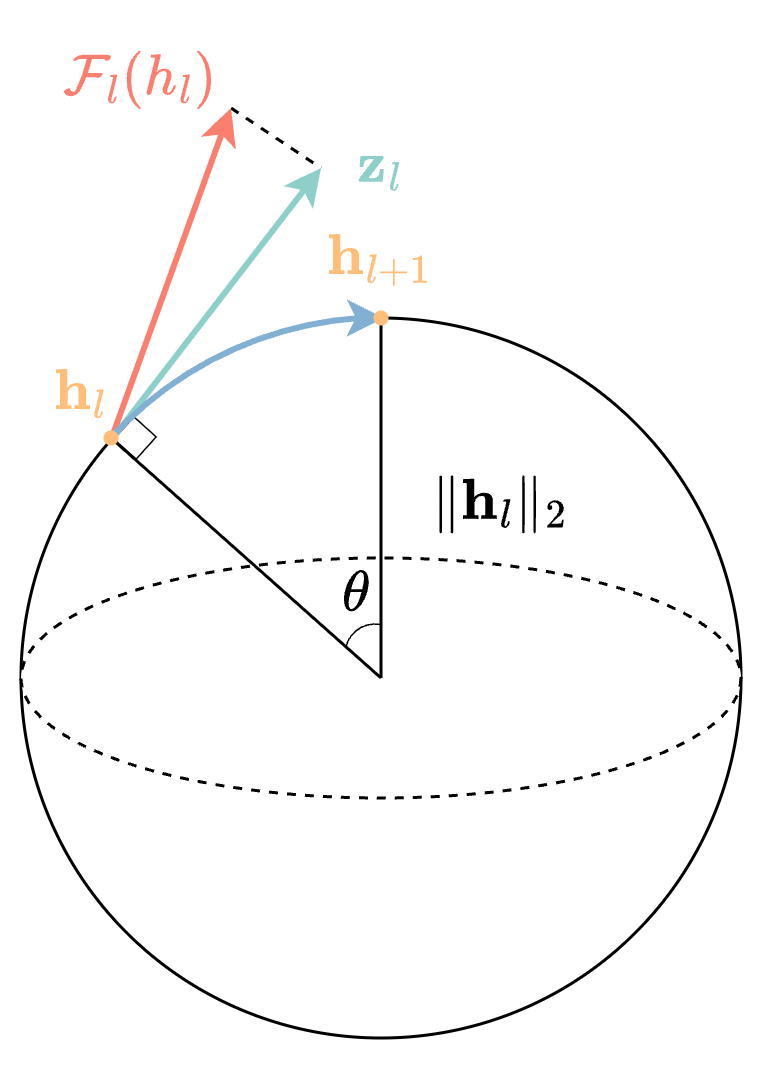}
    \caption{Illustration of the geometric meaning of SpheretNorm. The retraction map $\mathcal{R}_{l}(\cdot)$ is defined on the hypersphere of radius $\lVert{\mathbf{h}_{l}}\rVert_{2}$.}
    \label{fig:retraction}
\end{subfigure}
\caption{Overview of SpheretNorm (including GeoNorm, Proj-SpheretNorm, Cay-SpheretNorm, and $p$-SpheretNorm). (a) The overall computational pipeline, and (b) its underlying geometric interpretation.}
\label{fig:spheretnorm}
\end{figure}

A recent line of work~\citep{li2025mixln,pmlr-v267-kim25u,chen2026post0layernorm} seeks to reconcile the Post-LN and Pre-LN paradigms within Euclidean space. \citet{zheng2026geonorm0} pursue a similar goal but take a distinct route: their Geodesic Normalization (GeoNorm) accomplishes the residual connection on a Riemannian manifold. Let $\boldsymbol{h}_{l}$ denote the $l$-th hidden state, and let $\mathcal{F}_{l}(\cdot)$ denote the corresponding (non-)linear mapping. Specifically, GeoNorm first orthogonalizes $\mathcal{F}_{l}(\boldsymbol{h}_{l})$ against $\boldsymbol{h}_{l}$ via a Gram-Schmidt process to obtain an intermediate tensor $\boldsymbol{z}_{l}$, and then applies an exponential map at $\boldsymbol{h}_{l}$ along $\boldsymbol{z}_{l}$ to produce $\boldsymbol{h}_{l+1}$. By keeping the $\ell_{2}$-norm of every hidden state constant across layers, GeoNorm in effect confines the hidden states to a hypersphere, which stabilizes forward propagation.

While GeoNorm offers a compelling geometric perspective, its update is tied to one specific choice: the exponential map. From the viewpoint of Riemannian geometry, however, the exponential map is merely one member of a broad family of retraction maps~\citep{10.5555/1557548,10.1137/100802529}. Thus, a natural question arises: \emph{which retraction should a deep connection use on the hypersphere, and how does this choice shape the behavior of a deep neural network?}

Rather than proposing another way to unify Post-LN and Pre-LN into a single normalization mechanism, we revisit retraction maps on the hypersphere and propose three \underline{Sphe}re \underline{Ret}raction \underline{Norm}alizations (SpheretNorms): \textbf{Proj-SpheretNorm}, \textbf{Cay-SpheretNorm}, and \textbfps{$p$-SpheretNorm}. We first unify residual connection and GeoNorm within a single framework, under which GeoNorm can be interpreted as a Riemannian residual connection on the hypersphere. Building upon this perspective, we further exploit the metric projection retraction and the Cayley retraction on the hypersphere to derive Proj-SpheretNorm and Cay-SpheretNorm, respectively. 

We then observe that both Proj-SpheretNorm and Cay-SpheretNorm can be subsumed into a unified formulation based on the exponential map with different effective angles. We refer to this generalized formulation as $p$-SpheretNorm with a positive parameter $p$, which recovers Proj-SpheretNorm at $p = 1$ and Cay-SpheretNorm at $p = 2$. Moreover, as $p \to 0^{+}$, $p$-SpheretNorm reduces to the identity mapping, while as $p \to +\infty$ with $\alpha_{l} \in (0, 1]$, it recovers GeoNorm. Through a single tunable parameter, $p$-SpheretNorm provides a principled and flexible family of normalization layers that bridges existing methods.

We evaluate three types of SpheretNorm using nanoGPT~\citep{karpathy2022nanogpt} as the backbone on both pretraining and downstream tasks. Proj-SpheretNorm achieves the best results, attaining the lowest pretraining and validation losses on the medium- and large-size nanoGPT models. Results on the downstream tasks also show that Proj-SpheretNorm achieves high accuracy. Taken together, these results indicate that Proj-SpheretNorm is the strongest of the three SpheretNorm variants across both pretraining and downstream evaluations.

\paragraph{Contributions.}
In conclusion, our key contributions are summarized as follows:
\begin{enumerate}
    \item We unify residual connections and GeoNorm as retraction-based updates, showing that GeoNorm amounts to a Riemannian residual connection driven by the exponential map on the hypersphere.
    \item We instantiate two alternative retractions on the hypersphere, the metric projection retraction and the Cayley retraction, which yield Proj-SpheretNorm and Cay-SpheretNorm.
    \item We generalize both instances into $p$-SpheretNorm, a single-parameter family that reduces to the identity mapping as $p \to 0^{+}$, recovers
    Proj-SpheretNorm at $p = 1$ and Cay-SpheretNorm at $p = 2$, and recovers GeoNorm as $p \to +\infty$.
    \item On nanoGPT at up to 36 layers and $6.55$\,B tokens, SpheretNorm attains the lowest pre-training and validation losses at 24 and 36 layers and the highest average zero-shot accuracy on both OpenWebText and FineWeb-Edu, outperforming Pre-LN, Pre-DyT, Peri-LN, Keel, and GeoNorm, while converging in every setting in which Pre-LN or Keel diverges.
    \end{enumerate}
\section{Related Work}
\paragraph{Normalization Methods.}
LayerNorm~\citep{ba2016layer} has served as the foundation of Transformers~\citep{NIPS2017_3f5ee243} since its inception. Later, RMSNorm~\citep{NEURIPS2019_1e8a1942}, a simplified variant of LayerNorm that drops the re-centering operation, became the de facto standard normalization. A closely related and contemporaneous normalization method, ScaleNorm~\citep{nguyen-salazar-2019-transformers}, differs mainly in that it normalizes by the $\ell_{2}$-norm rather than the root mean square, and replaces the per-dimension gain vector with a single tied scalar, making it more parameter-efficient.

\paragraph{Normalization-Free Normalization Functions.}
A noteworthy line of work replaces LayerNorm with a zero-centered, bounded, center-sensitive, monotonic element-wise non-linearity. \citet{Zhu_2025_CVPR} propose the Dynamic Tanh (DyT) function, \citet{stollenwerk-2026-mathematical} introduce the Dynamic Inverse Square Root Unit (DyISRU) function, and \citet{Chen_2026_CVPR} develop the Dynamic erf (Derf) function. By eliminating the reduction along the hidden dimension, these activation functions remove the synchronization bottleneck of LayerNorm and can be fused with the preceding linear layer, substantially reducing overhead while preserving the squashing behavior of LayerNorm.

\paragraph{Skip Connection, and its Variants.}
ResNet~\citep{10.1007/978-3-319-46493-0_38,he2016deep} ushered in the era of very deep neural networks. Later, \citet{pmlr-v161-bachlechner21a} proposed ReZero, a residual connection with a learnable step size $\alpha_{l}$ initialized at zero. \citet{Touvron_2021_ICCV} further enhanced ReZero by using a learnable vector instead of a scalar. Recently, \citet{zheng2026geonorm0} proposed GeoNorm, a Riemannian residual connection that preserves the $\ell_{2}$-norm of each hidden state throughout forward propagation. Inspired by the property of GeoNorm that maintains a constant $\ell_{2}$-norm across hidden layers, we revisit retraction maps on the hypersphere, and propose SpheretNorms.

\paragraph{Post-LN, Pre-LN, and their Variants.}
With the advent of Transformers, applying a residual connection followed by LayerNorm, known as Post-LN, became standard practice. However, Transformers using Post-LN typically rely on a carefully tuned learning-rate warm-up schedule. In contrast, Pre-LN applies LayerNorm before each sub-layer and adds the residual connection afterward. While more stable, Transformers with Pre-LN suffer from the curse of depth~\citep{sun2025curse}, where deeper layers contribute far less to learning and representation than earlier ones, effectively widening the network rather than deepening it. LayerNorm Scaling (LNS)~\citep{sun2025curse} attenuates the normalized hidden state by a depth-dependent factor before each sub-layer, mitigating the curse of depth. In addition to the conventional Pre-LN and Post-LN designs, Peri-LN~\citep{pmlr-v267-kim25u} and Keel~\citep{chen2026post0layernorm} apply LayerNorm twice to reconcile the two paradigms. Collectively, these works explore different placements and scalings of LayerNorm within the residual stream to stabilize and deepen Transformers, yet they all operate within the Euclidean geometry of the hidden representations.
\section{Background and Preliminary}
Let $\mathbf{h}_{l}$ denote the $l$-th hidden state, with $\mathbf{h}_{1}$ being the input feature (or the output of the embedding layer, for Transformers), and let $\mathcal{F}_{l}(\cdot)$ denote the $l$-th (non-)linear mapping in Euclidean space, such as multi-head self-attention~\citep{NIPS2017_3f5ee243} or a multi-layer perceptron. Unless otherwise specified, for notational convenience, all bias terms are omitted throughout this work. Moreover, we assume $\mathbf{h}_{l} \in \mathbb{R}^{d}$ and $\mathbf{h}_{l} \neq \mathbf{0}$. For $l = 1,\,2,\,\ldots,\,L-1$, we take two instances of deep connection methods for illustrating.

\paragraph{Residual Connection (RC).}
To facilitate the training of deep CNNs, \citet{10.1007/978-3-319-46493-0_38,he2016deep} propose ResNet, whose key innovation is the residual connection. It can be formulated as
\begin{equation}\label{eq:res_conn}
\mathbf{h}_{l+1} = \mathbf{h}_{l} + \mathcal{F}_{l}(\mathbf{h}_{l}).
\end{equation}

\paragraph{Geodesic Normalization (GeoNorm).}
The exponential map in Riemannian geometry is defined as 
\begin{equation}\label{eq:exp_map}
\exp_{\mathbf{h}}(\mathbf{v}) = \cos(\frac{\lVert{\mathbf{v}}\rVert_{2}}{\lVert{\mathbf{h}}\rVert_{2}})\mathbf{h} + \lVert{\mathbf{h}}\rVert_{2}\sin(\frac{\lVert{\mathbf{v}}\rVert_{2}}{\lVert{\mathbf{h}}\rVert_{2}})\frac{\mathbf{v}}{\lVert{\mathbf{v}}\rVert_{2}},
\end{equation}
with $\mathbf{h}^{\top}\mathbf{v}=0$. Building upon this exponential map, \citet{zheng2026geonorm0} propose GeoNorm, which is defined as
\begin{equation}\label{eq:geonorm}
\begin{split}
\mathbf{z}_{l} &= \Bigl(\mathbf{I} - \frac{\mathbf{h}_{l}\mathbf{h}_{l}^{\top}}{\lVert{\mathbf{h}_{l}}\rVert^{2}}\Bigr)\mathcal{F}_{l}(\mathbf{h}_{l}),\quad 
\beta_{l} = \alpha_{l}\frac{\lVert{\mathbf{z}_{l}}\rVert_{2}}{\lVert{\mathbf{h}_{l}}\rVert_{2}}, \\
\mathbf{h}_{l+1} &= \cos(\beta_{l})\mathbf{h}_{l} + \lVert{\mathbf{h}_{l}}\rVert_{2}\sin(\beta_{l})\frac{\mathbf{z}_{l}}{\lVert{\mathbf{z}_{l}}\rVert_{2}}.
\end{split}
\end{equation}
Here, $\alpha_{l} \in (0, 1]$ is a learnable parameter~\footnote{The range of $\alpha_{l}$ is not explicitly specified in \citep{zheng2026geonorm0}. Given that $\alpha_{l}$ acts as the step size, we deduce its range from this property.}.
\section{Proposed Method}\label{sec:Proposed Method}
\subsection{Sphere Retraction Normalizations}
In this subsection, we begin by casting the standard residual connection and GeoNorm into a common framework that exposes their shared computational structure and distinguishes their geometric components. We single out the retraction map for redesign and consider alternatives to the exponential map. Instantiating the framework with the metric projection and Cayley retractions yields Proj-SpheretNorm and Cay-SpheretNorm, respectively. We then introduce a unified one-parameter construction, termed $p$-SpheretNorm, which recovers these two methods as exact special cases.

\subsubsection{Retraction Maps on the Hypersphere}
In a deep neural network, the $l$-th block maps $\mathbf{h}_{l-1}$ to $\mathbf{h}_l$, for $l=1,\ldots,L$. Given a manifold $\mathcal{M}$, the tangent-space projection $\operatorname{proj}_{T_{\mathbf{h}}\mathcal{M}}$, a step size $\alpha_l$, and a retraction $\mathcal{R}_{\mathbf{h}}$, a deep connection can be written abstractly as
\begin{equation}\label{eq:deep_conn_manifold}
\mathbf{h}_l = \mathcal{R}_{\mathbf{h}_{l-1}}\!\Bigl[\alpha_{l}\cdot\operatorname{proj}_{T_{\mathbf{h}_{l-1}}\mathcal{M}}\bigl(\mathcal{F}_{l}(\mathbf{h}_{l-1})\bigr)\Bigr].
\end{equation}
Within this framework, RC and GeoNorm share the same computational skeleton and differ only in how the individual components are instantiated. In particular, GeoNorm corresponds to choosing the hypersphere as the manifold and the exponential map as the retraction, which makes it a Riemannian residual connection. Table~\ref{tab:rc_vs_geonorm} contrasts the two methods
component by component.
\begin{table}[!ht]
\centering
\caption{Residual Connection vs. Geodesic Normalization.}
\label{tab:rc_vs_geonorm}
\begin{tabular}{l|cc}
\toprule
{} 
    & RC 
    & GeoNorm \\
\midrule
$\mathcal{M}$ 
    & $\mathbb{R}^{d}$
    & $\mathbb{S}^{d-1}_{r}=\{\mathbf{h}\in\mathbb{R}^{d}:\lVert\mathbf{h}\rVert_{2}=r\}$ \\
$T_{\mathbf{h}}\mathcal{M}$ 
    & $\mathbb{R}^{d}$ 
    & $\{\mathbf{v}\in\mathbb{R}^{d}:\mathbf{h}^{\top}\mathbf{v}=0\}$ \\
$\operatorname{proj}_{T_{\mathbf{h}}\mathcal{M}}(\mathbf{u})$ 
    & $\mathbf{u}$
    & $\Bigl(\mathbf{I}-\dfrac{\mathbf{h}\mathbf{h}^{\top}}{\lVert\mathbf{h}\rVert_{2}^{2}}\Bigr)\mathbf{u}$ \\
$\alpha_{l}$ 
    & $1$
    & $\alpha_{l} \in (0,1]$ \\
$\mathcal{R}_{\mathbf{h}}(\mathbf{v})$ 
    & $\mathbf{h}+\mathbf{v}$
    & $\exp_{\mathbf{h}}(\mathbf{v})$ \\
\bottomrule
\end{tabular}
\end{table}

Table~\ref{tab:rc_vs_geonorm} shows that GeoNorm chooses the retraction map to be the exponential map on the hypersphere. The exponential map, however, is merely one instance of a broader class of retractions~\citep{10.5555/1557548,10.1137/100802529}. We replace the exponential map with other retractions keeps the framework intact while enriching the design space. We recall the definition of a retraction below.

\begin{definition}[Retraction~\citep{10.5555/1557548}]\label{def:retraction}
A \emph{retraction} on a manifold $\mathcal{M}$ is a smooth mapping
$\mathcal{R}: T\mathcal{M}\to\mathcal{M},\ (\mathbf{h},\mathbf{v})\mapsto\mathcal{R}_{\mathbf{h}}(\mathbf{v})$, such that for every $\mathbf{h}\in\mathcal{M}$ and every $\mathbf{v} \in T_{\mathbf{h}}\mathcal{M}$, $\mathcal{R}_{\mathbf{h}}(\mathbf{0}) = \mathbf{h},\ \frac{\diff}{\diff{t}}\,\mathcal{R}_{\mathbf{h}}(t\mathbf{v})\bigg\vert_{t=0} = \mathbf{v}$.
\end{definition}

\paragraph{Sphere Metric Projection Retraction Normalization.}
We first choose the metric projection retraction on the hypersphere~\citep{10.5555/1557548}, also known as the gnomonic retraction on the hypersphere~\citep{10.1137/100802529}, as the foundation for proposing a novel normalization layer, which we name the Sphere Metric Projection Retraction Normalization (Proj-SpheretNorm). The definition of the metric projection retraction on the hypersphere is described as blows.
\begin{definition}[Metric Projection Retraction~\citep{10.5555/1557548}]\label{def:proj_rectration}
Let $\mathcal{M} = \mathbb{S}_r^{n-1} = \{\mathbf{h}\in\mathbb{R}^{n}:\|\mathbf{h}\|_{2}=r\}$ be the hypersphere of radius $r > 0$, endowed with the Riemannian metric induced from the Euclidean inner product. For each $\mathbf{h}\in\mathcal{M}$, the tangent space is $T_{\mathbf{h}}\mathcal{M}=\{\mathbf{v}\in\mathbb{R}^{n}:\,\langle\mathbf{h},\mathbf{v}\rangle=0\}$, and the projective retraction is defined by
\begin{equation}\label{eq:proj_retraction}
\mathcal{R}^{\text{Proj}}_{\mathbf{h}}(\mathbf{v}) = r\cdot\frac{\mathbf{h}+\mathbf{v}}{\lVert{\mathbf{h}+\mathbf{v}}\rVert_2},\quad \mathbf{v} \in T_{\mathbf{h}}\mathcal{M}.
\end{equation}
Then $\mathcal{R}^{\text{Proj}}$ is a second-order retraction on $\mathcal{M}$; that is, for every $\mathbf{h} \in \mathcal{M}$ and every $\mathbf{v} \in T_{\mathbf{h}}\mathcal{M}$, the curve $c(t):=\mathcal{R}_{\mathbf{h}}(t\mathbf{v})$ satisfies $c(0)=\mathbf{h},\ c'(0)=\mathbf{v},\ \mathbf{P}_{\mathbf{h}}\left(c''(0)\right)=\mathbf{0}$, where $\mathbf{P}_{\mathbf{h}}:\mathbb{R}^{n}\to T_{\mathbf{h}}\mathcal{M}$ denotes the orthogonal projection onto the tangent space.
\end{definition}

Based on Eq.~\eqref{eq:deep_conn_manifold} and Eq.~\eqref{eq:proj_retraction}, we have
\begin{equation}\label{eq:proj_spheretnorm}
\begin{split}
\mathbf{z}_{l} &= \Bigl(\mathbf{I} - \frac{\mathbf{h}_{l}\mathbf{h}_{l}^{\top}}{\lVert{\mathbf{h}_{l}}\rVert^{2}}\Bigr)\mathcal{F}_{l}(\mathbf{h}_{l}), \\
\mathbf{h}_{l+1} &= \lVert\mathbf{h}_{l}\rVert_{2}\cdot\frac{\mathbf{h}_{l} + \alpha_{l}\mathbf{z}_{l}}{\lVert{\mathbf{h}_{l} + \alpha_{l}\mathbf{z}_{l}}\rVert_{2}}.
\end{split}
\end{equation}



\paragraph{Sphere Cayley Retraction Normalization.}
We also consider the Cayley retraction on the hypersphere~\citep{10.5555/1557548}, also known as the stereographic retraction on the hypersphere~\citep{10.1137/100802529}, to propose Sphere Cayley Retraction Normalization (Cay-SpheretNorm). For a skew-symmetric matrix $\mathbf{W} \in \mathfrak{so}(d)$, where $\mathfrak{so}(d)=\left\{\boldsymbol{W}\in\mathbb{R}^{d\times d}:\boldsymbol{W}^\top=-\boldsymbol{W} \right\}$, the Cayley transform is defined by $\mathcal{C}(\mathbf{W}) = \left(\mathbf{I}-\frac{1}{2}\mathbf{W}\right)^{-1}\left(\mathbf{I}+\frac{1}{2}\mathbf{W}\right)$. On the hypersphere, a tangent vector determines a unique skew-symmetric generator supported on the two-dimensional subspace spanned by the current point and the tangent vector. This yields the following closed-form expression.

\begin{proposition}\label{prop:cayley-sphere}
Let $\mathbf{h}\in\mathbb{S}_r^{d-1}$ and $\mathbf{v}\in T_{\mathbf{h}}\mathbb{S}_r^{d-1}$. There exists a unique
matrix $\mathbf{W}=\mathbf{W}(\mathbf{h},\mathbf{v})\in\mathfrak{so}(d)$ satisfying $\mathbf{W}\mathbf{h}=\mathbf{v},\ \operatorname{range}(\mathbf{W}) \subseteq \operatorname{span}\{\mathbf{h},\mathbf{v}\}$,
and it is given by
\begin{equation}\label{eq:cayley-generator}
\mathbf{W}(\mathbf{h},\mathbf{v})
=\frac{\mathbf{v}\mathbf{h}^\top-\mathbf{h}\mathbf{v}^\top}{r^2}.
\end{equation}
Moreover, define $\rho=\frac{\|\mathbf{v}\|_2}{r}$. The corresponding Cayley retraction admits the closed form
\begin{equation}\label{eq:cayley-retraction}
    \mathcal{R}_{\mathbf{h}}^{\text{Cay}}(\mathbf{v})=\mathcal{C}\!\left(\mathbf{W}(\mathbf{h},\mathbf{v})\right)\mathbf{h}=\frac{4-\rho^2}{4+\rho^2}\mathbf{h}
    +\frac{4}{4+\rho^2}\mathbf{v}.
\end{equation}
\end{proposition}

The proof is provided in Appendix~\ref{app:cayley-sphere}. Let $\rho = \beta_{l}$, based on Eq.~\eqref{eq:deep_conn_manifold} and Eq.~\eqref{eq:cayley-retraction}, the whole process of Cay-SpheretNorm is defined as
\begin{equation}\label{eq:cay-spheretnorm}
\begin{split}
\mathbf{z}_{l} &= \Bigl(\mathbf{I} - \frac{\mathbf{h}_{l}\mathbf{h}_{l}^{\top}}{\lVert{\mathbf{h}_{l}}\rVert^{2}}\Bigr)\mathcal{F}_{l}(\mathbf{h}_{l}), \\
\mathbf{h}_{l+1} &= \frac{4-\beta_l^2}{4+\beta_l^2}\mathbf{h}_{l} + \frac{4\alpha_l}{4+\beta_l^2}\mathbf{z}_l.
\end{split}
\end{equation}

\paragraph{Sphere \texorpdfstring{$p$}{p}-Angular Retraction Normalization.}
The metric projection and Cayley retractions share the same
planar-rotation structure: they update the hidden state within
the two-dimensional plane spanned by the current state and the
projected tangent direction, differing only in how the normalized
step magnitude is converted into a rotation angle. Motivated by
this common structure, we introduce a one-parameter family of
hyper-spherical retraction normalizations, termed
$p$-SpheretNorm. The parameter $p>0$ controls the saturation
of the finite-step rotation angle while preserving the same local
behavior near the origin.

\begin{definition}[$p$-SpheretNorm]\label{def:p-spheretnorm}
Let parameter $p \in (0,+\infty)$, and the rotation angle $\theta_{l} = p\arctan\left(\frac{\beta_l}{p}\right)$, then the $p$-SpheretNorm is defined as
\begin{equation}
\begin{split}
\mathbf{z}_{l} &= \Bigl(\mathbf{I} - \frac{\mathbf{h}_{l}\mathbf{h}_{l}^{\top}}{\lVert{\mathbf{h}_{l}}\rVert^{2}}\Bigr)\mathcal{F}_{l}(\mathbf{h}_{l}),\\
\beta_{l} &= \alpha_{l}\frac{\lVert{\mathbf{z}_{l}}\rVert_{2}}{\lVert{\mathbf{h}_{l}}\rVert_{2}},\quad
\theta_{l} = p\arctan\left(\frac{\beta_l}{p}\right),\\
\mathbf{h}_{l+1} &= \cos(\theta_{l})\mathbf{h}_{l} + \lVert{\mathbf{h}_{l}}\rVert_{2}\sin(\theta_{l})\frac{\mathbf{z}_{l}}{\lVert{\mathbf{z}_{l}}\rVert_{2}}.
\end{split}
\end{equation}
\end{definition}

\begin{proposition}\label{prop:p-special-cases}
For every fixed $p>0$, the $p$-angular map underlying
$p$-SpheretNorm defines a smooth second-order retraction on
the hypersphere. Moreover, the following identities and limits hold:
\begin{enumerate}
    \item When $p=1$, $p$-SpheretNorm reduces exactly to Proj-SpheretNorm.
    \item When $p=2$, $p$-SpheretNorm reduces exactly to Cay-SpheretNorm.
    \item As $p\to +\infty$, the $p$-angular update converges
    to the exponential-map update used by GeoNorm.
    \item As $p\to 0^+$, the $p$-angular update converges
    point-wise to the identity mapping.
\end{enumerate}
\end{proposition}

To make the local geometry explicit, under the local boundedness
conditions stated in Appendix~\ref{app:p-angular-properties}, for
every fixed $p>0$, as $\alpha_l\to 0$,
\[
\mathbf{h}_{l+1}=\mathbf{h}_l+\alpha_l\mathbf{z}_l
-\frac{\alpha_l^2\|\mathbf{z}_l\|_2^2}{2r_l^2}\mathbf{h}_l
+\mathcal{O}(\alpha_l^3).
\]
The quadratic term is normal to the hypersphere at
$\mathbf{h}_l$; hence its tangent projection vanishes, which is
precisely the second-order retraction condition. In particular,
since $\mathbf{z}_l=\mathbf{P}_l\mathcal{F}_l(\mathbf{h}_l)$, we obtain
\[
\mathbf{h}_{l+1}=\mathbf{h}_l+\alpha_l\mathbf{P}_l\mathcal{F}_l(\mathbf{h}_l)
+\mathcal{O}(\alpha_l^2).
\]
Thus, all members of the $p$-angular family share the same
first-order tangent dynamics and the same second-order normal
correction, while the dependence on $p$ first appears at cubic
order. The complete proof, including smoothness at
$\mathbf{z}_l=\mathbf{0}$, is provided in
Appendix~\ref{app:p-angular-properties}.

\subsection{Implementation Details}

For the $p$-SpheretNorm, the rotation angle $\theta_{l}$ in each block is always bounded in $(-\frac{p\pi}{2},\,\frac{p\pi}{2})$. \citet{zheng2026geonorm0} suggest that the rotation angle in each block should be restricted to $(0,\,\frac{\pi}{2})$, and that restricting it to $(0,\,\frac{\pi}{4}]$ produces better results than restricting it to $(0,\,\frac{\pi}{2})$. Our experiments also support this view. We find that the rotation angle in each block should be restricted to $(0,\,\frac{\pi}{4}]$ for Cay-SpheretNorm and for $p$-SpheretNorm with $p > 1$; otherwise, training becomes unstable and may even suffer from exploding gradients. Therefore, for $p$-SpheretNorm, we choose $p = 0.5$, in which case the rotation angle in each block is kept in $(-\frac{\pi}{4},\,\frac{\pi}{4})$. 

For all three types of SpheretNorm, we further adopt two additional approaches that ensure training stability with fewer loss spikes. First, we adopt a decay method similar to the one in \citep{zheng2026geonorm0}: we let $\alpha_l = \lambda_{l}\cdot\operatorname{Softplus}(a_{l})$, where $\lambda_l \in (0,\,1)$ is a decay factor and $a_l$ is a learnable parameter. We find that setting $\lambda_{l} = 1 / \sqrt{l}$, initializing $a_{l} = \ln(\eu - 1)$, and constraining $\operatorname{Softplus}(a_{l}) \in (0, 1]$ can ensure training stability with fewer loss spikes. Second, we apply a modified ScaleNorm before the first (non-)linear mapping, i.e.,
\begin{equation}\label{eq:scalenorm_begin} 
\mathbf{h}_{1} \leftarrow \operatorname{Softplus}(\gamma) \cdot \frac{\mathbf{h}_{1}}{\lVert{\mathbf{h}_{1}}\rVert_{2}}, 
\end{equation} 
where $\gamma \in \mathbb{R}$ is a learnable parameter that is initialized as $\gamma = \sqrt{d}$ when $\sqrt{d} \geq 20$, and as $\gamma = \ln(\eu^{\sqrt{d}} - 1)$ otherwise. We use the $\texttt{torch.clamp}$ function to ensure that $\operatorname{Softplus}(\gamma) \in [1, \sqrt{d}]$, which we find can accelerate model convergence while stabilizing training.

\subsection{Analysis}\label{sec:analysis}
Our analysis of SpheretNorms focuses on four parts. First, we show that $p$-angular retraction is a second-order retraction. Second, we analyze the Jacobian to establish training stability. Third, we prove that the modified ScaleNorm in the first layer preserves the hypersphere. Fourth, we discuss the trade-off in choosing different decay factors. For analysis, we define $r_l=\|\mathbf{h}_{l}\|_2,\ \mathbf{P}_l=I-\frac{\mathbf{h}_{l}\mathbf{h}_{l}^\top}{r_l^2},
\ \mathbf{z}_l=\mathbf{P}_l\mathcal{F}_l(\mathbf{h}_{l}),\ s_l=\|\mathbf{z}_l\|_2$.

\paragraph{\texorpdfstring{$p$}{p}-SpheretNorm based on the second order retraction.}
Let $\eta_l=\frac{s_l}{r_l},\ \beta_l=\alpha_l\eta_l,\ \theta_l
=p\arctan\left(\frac{\beta_l}{p}\right)$. For every fixed $p>0$, the $p$-angular retraction is a second-order retraction. Under the local boundedness conditions stated in Appendix~\ref{app:p-angular-properties},
it also satisfies, as $\alpha_l\to 0$,
\[
\mathbf{h}_{l+1}=\mathbf{h}_{l}+\alpha_l\mathbf{P}_l\mathcal{F}_l(\mathbf{h}_{l})+\mathcal{O}(\alpha_l^2).
\]
The proof is provided in Appendix~\ref{app:p-angular-properties}. Thus, all members of the $p$-angular family share the same first-order
tangent dynamics and differ only through their higher-order
finite-step angular responses.

\paragraph{Jacobian conditioning from signal propagation.}
Exact norm preservation controls the activation scale, but
does not by itself control perturbation or gradient
propagation. Define $A_l=\frac{\partial \mathcal{F}_l(\mathbf{h}_{l})}{\partial \mathbf{h}_{l}},
\ c_l=\frac{\mathbf{h}_{l}^{\top}\mathcal{F}_l(\mathbf{h}_{l})}{r_l^2},
\ \mathbf{B}_l=\mathbf{A}_l-c_l\mathbf{I}$,
and let $\mathbf{J}_l=\frac{\partial \mathbf{h}_{l+1}}{\partial \mathbf{h}_{l}}$. For $s_l>0$, define $Q_l=\cos\left(\theta_l\right)\mathbf{I}+
\frac{\sin\left(\theta_l\right)}{r_ls_l}
\left(\mathbf{z}_l\mathbf{h}_{l}^{\top}-\mathbf{h}_{l}\mathbf{z}_l^{\top}\right)$.

\begin{proposition}[Layer-wise and depth-wise Jacobian bounds]
\label{prop:jacobian-bounds}
For $s_l>0$, the Jacobian of the $l$-th SpheretNorm block admits the exact decomposition $\mathbf{J}_l=\mathbf{Q}_l+\mathbf{E}_l$, where $\|\mathbf{Q}_l\|_2=1,\ \sigma_{\min}(\mathbf{Q}_l)\geq\left|\cos\left(\theta_l\right)\right|$.
Define $\Delta_l=|\alpha_l|\left(2\|\mathbf{B}_l\|_2+\eta_l\right)$, $q_l=\left|\cos\left(\theta_l\right)\right|$, and $m_l=\max\{0,q_l-\Delta_l\}$. Then $\|\mathbf{E}_l\|_2\leq\Delta_l$, and, 
\[
\sigma_{\max}(\mathbf{J}_l)\leq 1+\Delta_l,\quad
\sigma_{\min}(\mathbf{J}_l)\geq m_l.
\]
Consequently, for $\mathbf{G}_L=\frac{\partial \mathbf{h}_L}{\partial \mathbf{h}_1}
=\mathbf{J}_L \mathbf{J}_{L-1}\cdots \mathbf{J}_1$, one has
\[
\sigma_{\max}(\mathbf{G}_L)\leq\exp\left(\sum_{l=1}^{L}\Delta_l\right),
\quad\sigma_{\min}(\mathbf{G}_L)\geq\prod_{l=1}^{L}m_l.
\]
\end{proposition}

The full Jacobian calculation and the smooth $s_l=0$ extension are provided in Appendix~\ref{app:full-jacobian}. The angle $\theta_l$ controls the geometric
component $\mathbf{Q}_l$, whereas $|\alpha_l|\|\mathbf{B}_l\|_2$ and $|\beta_l|$ control the block-induced perturbation. Hence, a cap
$$
|\theta_l|\leq\theta_{\max}<\frac{\pi}{2}
$$
keeps $\mathbf{Q}_l$ away from singularity but does not by itself control the full Jacobian; a counterexample is given in Appendix~\ref{app:angle-control-insufficient}. Moreover, if $m_l\geq m_*>0$, then
\[
\sum_{l=1}^{L}\Delta_l=\mathcal{O}(1),
\quad\sum_{l=1}^{L}\left|\theta_l\right|^2=\mathcal{O}(1)
\]
is a sufficient condition for depth-independent
conditioning; see Appendix~\ref{app:depthwise-conditioning}.

\paragraph{First-layer modified ScaleNorm preserves hypersphere.}
\begin{proposition}[Fixed-radius spherical dynamics]
\label{prop:fixed-radius-dynamics}
Recall Eq.~\eqref{eq:scalenorm_begin}. For $l=1,\,2,\,\ldots,\,L$, suppose that $\mathbf{h}_{l+1}$ is generated from $\mathbf{h}_{l}$ by a SpheretNorm update. Then $\lVert{\mathbf{h}_l}\rVert_{2}=\operatorname{Softplus}(\gamma)$. Consequently, $\lVert{\mathbf{h}_{L}}\rVert_{2}\leq\operatorname{Softplus}(\gamma)$.
\end{proposition}

The proof is provided in Appendix~\ref{app:fixed-radius-dynamics}. Thus, the first modified ScaleNorm selects the hypersphere $\mathbb{S}^{d-1}_{\operatorname{Softplus}(\gamma)}=\{\mathbf{h}\in\mathbb{R}^d:\|\mathbf{h}\|_2=\operatorname{Softplus}(\gamma)\}$, while all subsequent SpheretNorm blocks preserve this
hypersphere exactly. For Proj-SpheretNorm, this relation can be made more
explicit. By Proposition~\ref{prop:fixed-radius-dynamics},
$\lVert{\mathbf{h}_{l}}\rVert_{2} = \operatorname{Softplus}(\gamma)$, and therefore
\begin{align*}
\mathbf{h}_{l+1}
&= \operatorname{Softplus}(\gamma)\cdot\frac{\mathbf{h}_{l}+\alpha_l \mathbf{z}_l}{\|\mathbf{h}_{l}+\alpha_l \mathbf{z}_l\|_2}\\
&= \Pi_{\operatorname{Softplus}(\gamma)}\bigl(\mathbf{h}_{l}+\alpha_l\mathbf{P}_l\mathcal{F}_l(\mathbf{h}_{l})\bigr),    
\end{align*}
where $\Pi_{\operatorname{Softplus}(\gamma)}(\mathbf{u})=\operatorname{Softplus}(\gamma)\frac{\mathbf{u}}{\|\mathbf{u}\|_2}$ is the radial projection onto $\mathbb{S}^{d-1}_{\operatorname{Softplus}(\gamma)}$.

\paragraph{Training stability trade-off from different decay factors.}
The prescribed factor $\lambda_l$ remains fixed throughout
training. Since $0<a_l\leq1$, we have $0<\alpha_l=\lambda_la_l\leq\lambda_l$. Let $\kappa_l=2\|\mathbf{B}_l\|_2+\eta_l$,
$\mathcal{S}_L=\sum_{l=1}^{L}\Delta_l=\sum_{l=1}^{L}\lambda_la_l\kappa_l$, and $\mathcal{A}_L=\sum_{l=1}^{L}\left|\theta_l\right|^2$.

\begin{proposition}[Growth under depth schedules]
\label{prop:depth-schedule-growth}
Suppose that there exist constants $K,\,H>0$, independent of $l$ and $L$, such that $\kappa_l\leq K$, $\ a_l\eta_l\leq H$. Then
\[
\mathcal{S}_L\leq K\sum_{l=1}^{L}\lambda_l,
\quad\mathcal{A}_L\leq H^2\sum_{l=1}^{L}\lambda_l^2.
\]
Consequently,
\[
\begin{array}{c|cc}
\lambda_l 
    & \mathcal{S}_L 
    & \mathcal{A}_L \\ \hline
1 
    & \mathcal{O}(L) 
    & \mathcal{O}(L) \\
1/\sqrt{l} 
    & \mathcal{O}(\sqrt{L}) 
    & \mathcal{O}(\log L) \\
\sqrt{l} 
    & \mathcal{O}(\log L) 
    & \mathcal{O}(1)
\end{array}
\]
\end{proposition}

The proof is provided in Appendix~\ref{app:depth-schedule-growth}. Thus, the decay factor $\lambda_{l} = 1 / l$ gives the more conservative global worst-case bounds. However, Proposition~\ref{prop:depth-schedule-growth} only provides global upper bounds. A finer dyadic analysis is given in Appendix~\ref{app:dyadic-angular-budget}. At any fixed optimization iteration, under the bounded small-step and non-degeneracy assumptions stated there, the accumulated squared angular motion over each dyadic interval $I_m$ has the same asymptotic order as $\sum_{l\in I_m}\lambda_l^2$. Hence, the square-root schedule assigns an approximately constant squared angular budget to each logarithmic depth scale, whereas the the decay factor $\lambda_{l} = 1 / l$ assigns a budget of order $m^{-1}$.

The empirically preferred the decay factor $\lambda_{l} = 1 / \sqrt{l}$ should be interpreted as a finite-depth compromise: it provides a tighter Jacobian envelope than the non-decaying schedule while preserving stronger transformations in deeper layers than the decay factor $\lambda_{l} = 1 / l$. This mechanism is consistent with the empirical advantage of the decay factor $\lambda_{l} = 1 / \sqrt{l}$, but does not imply uniformly better worst-case conditioning or a guaranteed lower loss.

\begin{table*}[t]
\centering
\caption{Loss of trained models. We report training and validation loss at the end of training. To mitigate stochastic fluctuations, training loss is computed as a moving average over the last 200 iterations.}
\label{tab:pretrain_loss}
{
\setlength{\tabcolsep}{4pt}
\begin{tabular}{lcc cc cc  cc cc cc}
\toprule
Dataset
    & \multicolumn{6}{c}{FineWeb-Edu}
    & \multicolumn{6}{c}{OpenWebText} \\
\cmidrule(lr){2-7}\cmidrule(lr){8-13}
Model Scale
    & \multicolumn{2}{c}{\sf{S}} & \multicolumn{2}{c}{\sf{M}} & \multicolumn{2}{c}{\sf{L}}
    & \multicolumn{2}{c}{\sf{S}} & \multicolumn{2}{c}{\sf{M}} & \multicolumn{2}{c}{\sf{L}} \\
\cmidrule(lr){2-3}\cmidrule(lr){4-5}\cmidrule(lr){6-7}
\cmidrule(lr){8-9}\cmidrule(lr){10-11}\cmidrule(lr){12-13}
& Train & Val & Train & Val & Train & Val
& Train & Val & Train & Val & Train & Val \\
\midrule
Pre-LN & $\mathbf{3.057}$ & $\mathbf{3.071}$ & $2.975$ & $2.976$ & $2.842$ & $2.866$ & $\mathbf{3.062}$ & $\mathbf{3.061}$ & $4.697$ & $4.716$ & $2.858$ & $2.879$ \\
Pre-DyT & $3.626$ & $3.619$ & $3.025$ & $3.019$ & $2.845$ & $2.870$ & $3.539$ & $3.556$ & $3.012$ & $3.035$ & $2.843$ & $2.866$ \\
Peri-LN & $3.091$ & $3.088$ & $2.952$ & $2.949$ & $2.817$ & $2.844$ & $3.071$ & $3.089$ & $2.950$ & $2.973$ & $2.828$ & $2.850$ \\
Keel & $7.645$ & $7.641$ & $3.006$ & $3.001$ & $2.941$ & $2.959$ & $3.406$ & $3.422$ & $2.997$ & $3.018$ & $2.945$ & $2.964$ \\
GeoNorm & $3.176$ & $3.185$ & $3.096$ & $3.088$ & $2.940$ & $2.958$ & $3.209$ & $3.206$ & $3.018$ & $3.039$ & $3.001$ & $3.017$ \\
\midrule
$0.5$-SpheretNorm & $\underline{3.089}$ & $\underline{3.086}$ & $2.831$ & $2.828$ & $\underline{2.809}$ & $\underline{2.833}$ & $3.067$ & $3.085$ & $\underline{2.932}$ & $\underline{2.955}$ & $\underline{2.826}$ & $\underline{2.847}$ \\
$1$-SpheretNorm & $3.098$ & $3.109$ & $\mathbf{2.825}$ & $\mathbf{2.822}$ & $\mathbf{2.808}$ & $\mathbf{2.832}$ & $3.064$ & $3.082$ & $\mathbf{2.931}$ & $\mathbf{2.954}$ & $\mathbf{2.795}$ & $\mathbf{2.822}$ \\
$2$-SpheretNorm & $3.119$ & $3.114$ & $\underline{2.826}$ & $\underline{2.824}$ & $2.809$ & $2.834$ & $\underline{3.063}$ & $\underline{3.081}$ & $\mathbf{2.931}$ & $\underline{2.955}$ & $2.841$ & $2.861$ \\
\bottomrule
\end{tabular}
}
\end{table*}

\begin{table}[t]
\centering
\caption{Perplexity of trained models. We report validation perplexity at the end of training.}
\label{tab:pretrain_ppl}
{
\small
\setlength{\tabcolsep}{2pt}
\begin{tabular}{@{}l ccc @{\hspace{6pt}} ccc@{}}
\toprule
Dataset & \multicolumn{3}{c}{FineWeb-Edu} & \multicolumn{3}{c}{OpenWebText} \\
\cmidrule(lr){2-4}\cmidrule(lr){5-7}
Scale   & \sf{S} & \sf{M} & \sf{L} & \sf{S} & \sf{M} & \sf{L} \\
\midrule
Pre-LN  & $\mathbf{21.56}$ & $19.61$ & $17.57$ & $\mathbf{21.35}$ & $111.70$ & $17.80$ \\
Pre-DyT & $37.29$ & $20.48$ & $17.64$ & $35.03$ & $20.79$ & $17.56$ \\
Peri-LN & $21.94$ & $19.09$ & $17.18$ & $21.96$ & $19.55$ & $17.29$ \\
Keel    & $2081$ & $20.11$ & $19.27$ & $30.63$ & $20.46$ & $19.37$ \\
GeoNorm & $24.16$ & $21.93$ & $19.27$ & $24.67$ & $20.88$ & $20.42$ \\
\midrule
\multicolumn{7}{@{}l}{\textit{SpheretNorm (ours)}} \\
\quad $0.5$ & $\underline{21.88}$ & $16.91$ & $\underline{16.99}$ & $21.86$ & $\underline{19.20}$ & $\underline{17.24}$ \\
\quad $1$   & $22.40$ & $\mathbf{16.81}$ & $\mathbf{16.98}$ & $21.81$ & $\mathbf{19.18}$ & $\mathbf{16.81}$ \\
\quad $2$   & $22.52$ & $\underline{16.84}$ & $17.01$ & $\underline{21.79}$ & $\underline{19.20}$ & $17.48$ \\
\bottomrule
\end{tabular}
}
\end{table}
\section{Experiments}
\paragraph{Experimental Settings.}
To evaluate SpheretNorms against other deep connection methods, including Pre-LN~\citep{wang-etal-2019-learning-deep, pmlr-v119-xiong20b}, Pre-DyT~\citep{Zhu_2025_CVPR}, Peri-LN~\citep{pmlr-v267-kim25u}, Keel~\citep{chen2026post0layernorm}, and GeoNorm~\citep{zheng2026geonorm0}, we implement all baselines within the nanoGPT framework~\citep{karpathy2022nanogpt}. For pre-training, we adopt two datasets: OpenWebText~\citep{Gokaslan2019OpenWeb} and FineWeb-Edu~\citep{NEURIPS2024_370df50c}. All models are optimized with the AdamW optimizer~\citep{loshchilov2018decoupled} under a cosine annealing learning rate schedule with a linear warm-up. We employ \texttt{bfloat16} mixed-precision training together with gradient clipping. We conducted all experiments on a single node equipped with two 56-core Intel Xeon Platinum 8480+ CPUs, 2\,TB of RAM, and eight NVIDIA H100 SXM GPUs, each with 80\,GB of memory. The code is implemented in PyTorch~\citep{NEURIPS2019_bdbca288} using its Distributed Data Parallel (DDP) with the NCCL backend. We train each model for 50,000 iterations (about $6.55$\,B tokens in total). We fix the random seed to 1337 in all experiments. Detailed hyperparameter configurations are reported in Appendix~\ref{app:hyperparameters}.

\subsection{Pre-Training Language Models}
To assess whether SpheretNorm attains loss improvements comparable to existing residual connection designs, we compare the final training and validation losses in Table~\ref{tab:pretrain_loss} and the corresponding validation perplexities in Table~\ref{tab:pretrain_ppl}. At the small scale (12 layers, $\sim\!0.12$\,B parameters), Pre-LN attains the lowest training and validation loss on both FineWeb-Edu and OpenWebText; the best SpheretNorm variant stays within $0.02$ nats of it ($21.88$ vs.\ $21.56$ validation perplexity on FineWeb-Edu, $21.79$ vs.\ $21.35$ on OpenWebText) and is on par with Peri-LN. This ordering changes as depth increases. At the medium scale (24 layers, $\sim\!0.35$\,B), $1$-SpheretNorm attains the lowest training and validation loss on both datasets, improving validation perplexity over the strongest baseline, Peri-LN, from $19.09$ to $16.81$ on FineWeb-Edu and from $19.55$ to $19.18$ on OpenWebText. The same holds at the large scale (36 layers, $\sim\!0.77$\,B), where Proj-SpheretNorm reaches $16.98$ and $16.81$ validation perplexity on the two datasets, ahead of Peri-LN at $17.18$ and $17.29$. $2$-SpheretNorm and $0.5$-SpheretNorm track $1$-SpheretNorm closely in every setting, with all three variants falling within $0.04$ nats of one another, indicating that the benefit is not sensitive to the particular choice of $p$. The methods also differ in stability: training destabilizes for Pre-LN on OpenWebText at the medium scale (validation perplexity $111.70$) and for Keel on FineWeb-Edu at the small scale ($2081.50$), whereas all three SpheretNorm variants converge in all six settings. We show the training and validation losses in FineWeb-Edu in Figure~\ref{fig:loss_fwe}.

\subsection{Downstream Evaluations}
We evaluate zero-shot performance using the lm-evaluation-harness library~\citep{eval-harness} on eleven benchmarks: ARC-Challenge, ARC-Easy~\citep{clark2018thinksolvedquestionanswering}, BoolQ~\citep{clark-etal-2019-boolq}, HellaSwag~\citep{zellers-etal-2019-hellaswag}, OpenBookQA~\citep{mihaylov-etal-2018-suit}, PIQA~\citep{bisk2020piqa}, SciQ~\citep{welbl-etal-2017-crowdsourcing}, SocialIQa~\citep{sap-etal-2019-social}, WinoGrande~\citep{10.1145/3474381}, LAMBADA~\citep{paperno-etal-2016-lambada}, and WikiText-2~\citep{merity2017pointer}. Tables~\ref{tab:downstream_benchmark_fwe} and~\ref{tab:downstream_benchmark_owt} report the large-size models trained for 6.55\,B tokens on FineWeb-Edu and OpenWebText, respectively.

SpheretNorms attain the best score on 11 of the 12 metrics on FineWeb-Edu and 10 of 12 on OpenWebText; the exceptions are SocialIQa on both corpora and HellaSwag on OpenWebText, where the leading baseline is ahead by at most $0.52$ points. $1$-SpheretNorm is the strongest variant overall, though the three stay within about one point of each other on most tasks. The gains are clearest in perplexity, the metric least affected by the noise of zero-shot evaluation at this scale: on FineWeb-Edu every variant improves over every baseline, with $1$-SpheretNorm reaching $46.74$ on LAMBADA against $59.55$ for the best baseline ($21.5\%$ relative) and $33.70$ on WikiText-2; on OpenWebText it reaches $31.13$ and $30.79$.

Two baselines warrant comment. Keel, a Post-LN design targeting gradient flow at extreme depth, brings no benefit in our shallow, short-budget regime and is among the weakest configurations. GeoNorm is the closest comparison, sharing our view of the hidden state as a point on a sphere and of each sublayer output as an update direction, but taking the exact geodesic step via the exponential map where SpheretNorm uses a retraction. The retraction is the better choice here: our best variant improves over GeoNorm on all twelve metrics on both corpora, by margins far larger than those separating the other methods. Since several benchmarks remain near chance at this scale (e.g., WinoGrande, SocialIQa), our conclusion rests on the perplexities and on the consistency across the two corpora rather than on individual sub-point differences.

\begin{table*}[t]
\centering
\caption{Zero-shot evaluation results for large-size models trained on FineWeb-Edu for 6.55\,B tokens and evaluated with lm-evaluation-harness. acc means accuracy, acc\_n means length-normalized accuracy, and ppl means perplexity.}
\label{tab:downstream_benchmark_fwe}
\resizebox{\linewidth}{!}{
\begin{tabular}{l|cccccccccc|cc}
\toprule
nanoGPT w/ 
    & ARC-C
    & ARC-E
    & BoolQ
    & HellaSwag
    & LMB.
    & OBQA
    & PIQA
    & SciQ
    & SocialIQa
    & WinoGrande
    & LMB.
    & Wiki. \\
{}
    & acc\_n\,$\uparrow$ 
    & acc\_n\,$\uparrow$ 
    & acc\,$\uparrow$
    & acc\_n\,$\uparrow$
    & acc\,$\uparrow$
    & acc\_n\,$\uparrow$
    & acc\_n\,$\uparrow$
    & acc\,$\uparrow$
    & acc\,$\uparrow$
    & acc\,$\uparrow$
    & ppl\,$\downarrow$
    & ppl\,$\downarrow$ \\
\midrule
Pre-LN & $26.54$ & $51.85$ & $55.81$ & $36.81$ & $29.11$
& $33.40$ & $65.18$ & $81.20$ & $37.56$ & $52.49$ & $66.69$ & $36.23$ \\
Pre-DyT & $27.39$ & $51.14$ & $60.34$ & $36.64$ & $29.52$
& $31.40$ & $65.02$ & $\underline{82.50}$ & $\mathbf{38.28}$ & $50.75$ & $59.55$ & $35.52$ \\
Peri-LN & $26.62$ & $50.93$ & $53.27$ & $\underline{38.09}$ & $29.01$
& $31.20$ & $64.85$ & $\mathbf{82.60}$ & $\underline{38.13}$ & $52.33$ & $67.09$ & $34.77$ \\
Keel & $26.45$ & $49.24$ & $61.59$ & $34.39$ & $26.97$
& $30.60$ & $64.25$ & $80.90$ & $38.08$ & $48.62$ & $81.06$ & $39.73$ \\
GeoNorm & $27.82$ & $48.02$ & $\underline{62.05}$ & $34.35$ & $27.98$
& $30.20$ & $63.93$ & $80.50$ & $36.80$ & $51.78$ & $80.64$ & $39.88$ \\
\midrule
$0.5$-SpheretNorm & $\underline{27.90}$ & $50.93$ & $60.87$ & $37.77$ & $30.84$
& $\underline{33.80}$ & $\underline{66.27}$ & $81.60$ & $37.69$ & $\mathbf{52.88}$ & $\underline{50.00}$ & $\underline{33.74}$ \\
$1$-SpheretNorm & $\mathbf{28.67}$ & $\mathbf{52.78}$ & $\mathbf{62.70}$ & $\mathbf{38.21}$ & $\mathbf{32.93}$
& $32.60$ & $66.05$ & $82.10$ & $\underline{38.13}$ & $51.36$ & $\mathbf{46.74}$ & $\mathbf{33.70}$ \\
$2$-SpheretNorm & $\underline{27.90}$ & $\underline{52.53}$ & $61.34$ & $38.07$ & $\underline{31.19}$
& $\mathbf{34.40}$ & $\mathbf{66.70}$ & $\mathbf{82.60}$ & $37.26$ & $\underline{52.72}$ & $53.45$ & $33.85$ \\
\bottomrule
\end{tabular}
}
\end{table*}

\begin{table*}[t]
\centering
\caption{Zero-shot evaluation results for large-size models trained on OpenWebText for 6.55\,B tokens and evaluated with lm-evaluation-harness. acc means accuracy, acc\_n means length-normalized accuracy, and ppl means perplexity.}
\label{tab:downstream_benchmark_owt}
\resizebox{\linewidth}{!}{
\begin{tabular}{l|cccccccccc|cc}
\toprule
nanoGPT w/ 
    & ARC-C
    & ARC-E
    & BoolQ
    & HellaSwag
    & LMB.
    & OBQA
    & PIQA
    & SciQ
    & SocialIQa
    & WinoGrande
    & LMB.
    & Wiki. \\
{}
    & acc\_n\,$\uparrow$ 
    & acc\_n\,$\uparrow$ 
    & acc\,$\uparrow$
    & acc\_n\,$\uparrow$
    & acc\,$\uparrow$
    & acc\_n\,$\uparrow$
    & acc\_n\,$\uparrow$
    & acc\,$\uparrow$
    & acc\,$\uparrow$
    & acc\,$\uparrow$
    & ppl\,$\downarrow$ 
    & ppl\,$\downarrow$ \\
\midrule
Pre-LN & $23.38$ & $40.49$ & $59.54$ & $32.73$ & $32.19$
& $26.20$ & $62.89$ & $69.90$ & $38.33$ & $51.54$ & $39.24$ & $33.19$ \\
Pre-DyT & $\underline{25.60}$ & $40.87$ & $54.31$ & $32.89$ & $33.92$
& $26.20$ & $62.13$ & $71.60$ & $\underline{38.54}$ & $51.54$ & $33.40$ & $34.11$ \\
Peri-LN & $25.43$ & $42.00$ & $59.17$ & $\mathbf{33.95}$ & $33.73$
& $25.80$ & $63.87$ & $73.20$ & $\mathbf{38.64}$ & $50.04$ & $31.48$ & $32.90$ \\
Keel & $23.21$ & $38.80$ & $57.61$ & $31.17$ & $30.70$
& $27.40$ & $61.97$ & $71.90$ & $38.13$ & $51.70$ & $44.43$ & $38.64$ \\
GeoNorm & $24.87$ & $37.84$ & $\underline{60.55}$ & $29.97$ & $30.00$
& $26.20$ & $61.92$ & $69.30$ & $37.67$ & $51.14$ & $51.91$ & $42.24$ \\
\midrule
$0.5$-SpheretNorm & $25.15$ & $\underline{42.21}$ & $59.88$ & $\underline{33.43}$ & $\underline{34.60}$
& $27.40$ & $63.44$ & $73.20$ & $37.92$ & $\mathbf{52.96}$ & $\underline{31.38}$ & $\underline{31.02}$ \\
$1$-SpheretNorm & $\mathbf{26.46}$ & $\mathbf{42.26}$ & $\mathbf{61.04}$ & $32.90$ & $\mathbf{34.99}$
& $\underline{27.80}$ & $\underline{63.98}$ & $\mathbf{73.60}$ & $37.36$ & $\underline{52.85}$ & $\mathbf{31.13}$ & $\mathbf{30.79}$ \\
$2$-SpheretNorm & $25.32$ & $42.05$ & $58.74$ & $32.86$ & $34.39$
& $\mathbf{28.60}$ & $\mathbf{64.20}$ & $\underline{73.40}$ & $38.48$ & $51.14$ & $31.50$ & $31.55$ \\
\bottomrule
\end{tabular}
}
\end{table*}
\section{Limitations}
While the proposed SpheretNorms are designed to stabilize the training of a broad range of deep neural networks, two limitations remain. First, the proposed methods are not directly applicable to certain specialized architectures, such as Equivariant Neural Networks~\citep{lim2022equivariantneuralnetwork}, for which the structural constraints are not naturally compatible with our formulation. Since SpheretNorms rely on $\ell_2$-norm and the associated projection, they commute with a group representation only when that representation preserves the norm, and they generally break the equivariance condition for non-orthogonal representations such as the $\mathrm{GL}(n)$ and Lorentz groups. Second, owing to limited research resources, we have not evaluated SpheretNorms on larger-scale deep Transformer models. We leave the extension to such architectures and the verification at larger scales to future work.

\begin{figure}[!ht]
  \centering
  \begin{subfigure}{0.32\textwidth}
    \includegraphics[width=\linewidth]{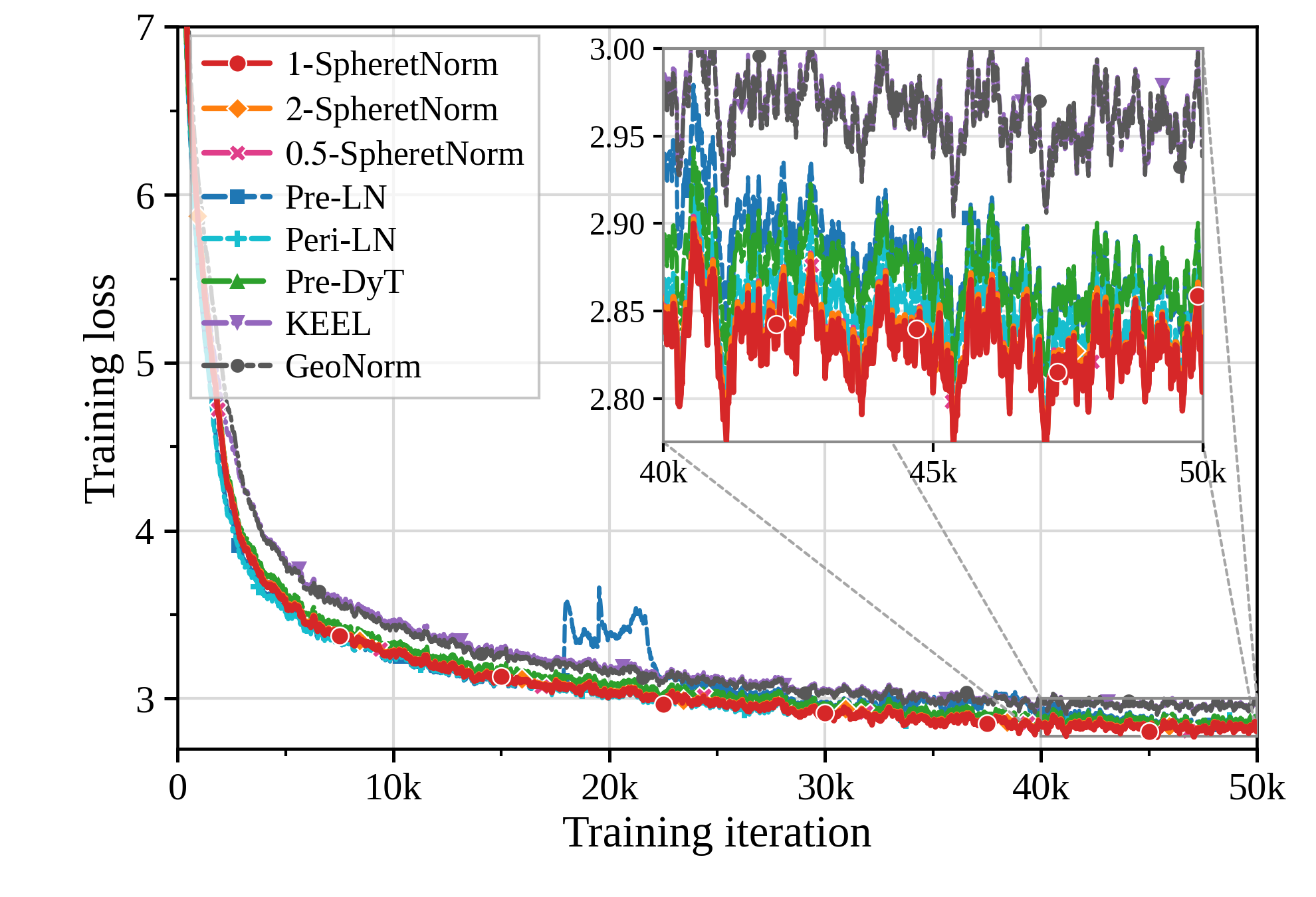}
    \caption{Training loss.}
  \end{subfigure}\hfill
  \begin{subfigure}{0.32\textwidth}
    \includegraphics[width=\linewidth]{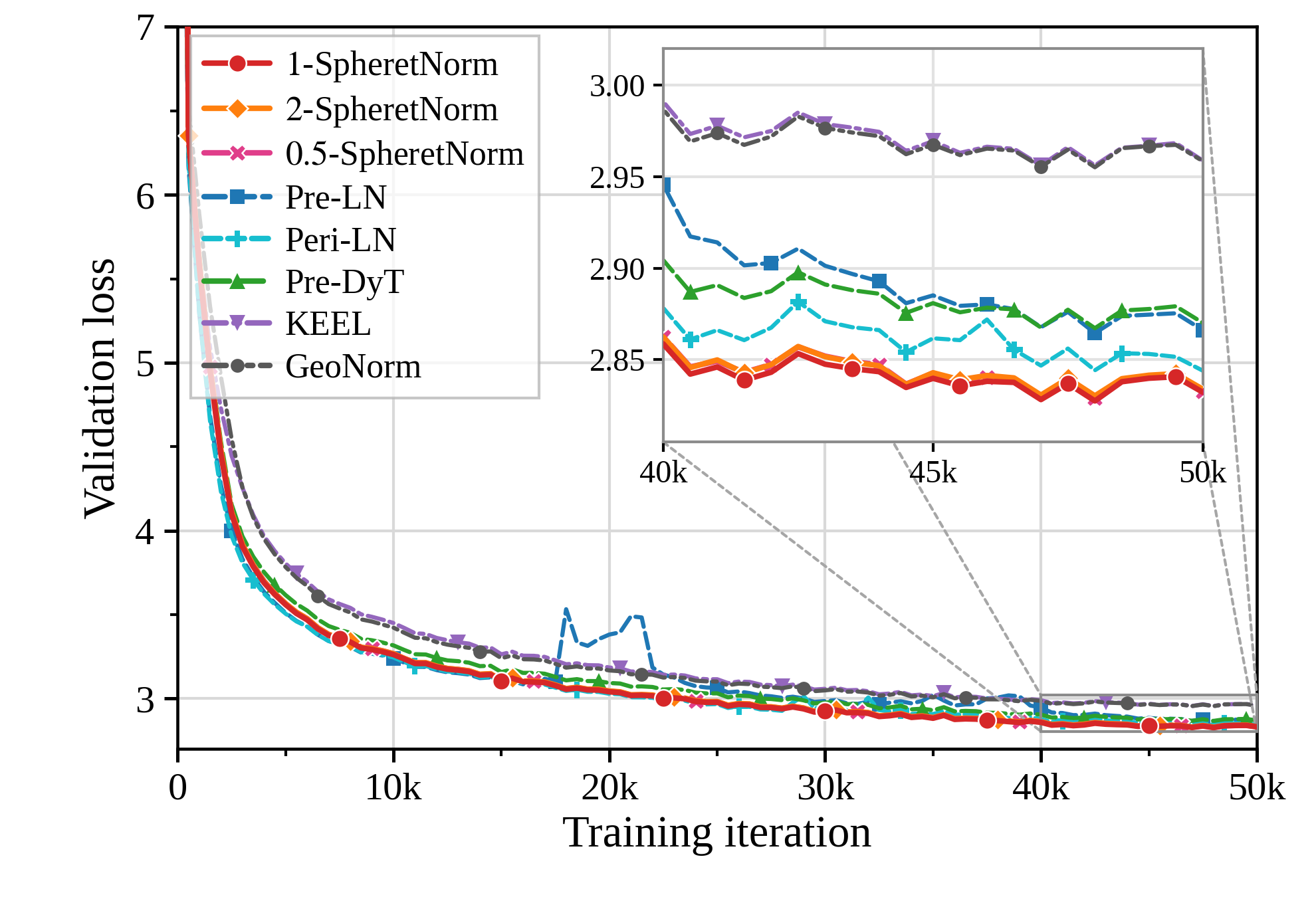}
    \caption{Validation loss.}
  \end{subfigure}
  \caption{Training and validation losses in the FineWeb-Edu.}
  \label{fig:loss_fwe}
\end{figure}

\section{Conclusion and Future Work}
In this work, we revisit retraction maps on the hypersphere and introduce three types of SpheretNorms: Proj-SpheretNorm, Cay-SpheretNorm, and $p$-SpheretNorm. We analyze the three SpheretNorms from both theoretical and practical perspectives. We implement them within the nanoGPT framework and evaluate them on pre-training and downstream tasks. The experimental results show that all three SpheretNorms achieve lower loss and higher accuracy than the baseline methods.

At a deeper level, designing deep connection mechanisms is fundamentally a mathematical problem at the intersection of graph theory, optimization, and dynamical systems, raising two intertwined questions: (i) which connection topologies enable non-trivial information propagation across depth while keeping the singular values of the end-to-end Jacobian bounded by constants independent of the depth $L$; and (ii) among such topologies, is there a unifying principle that identifies the optimal trade-off between information propagation and Jacobian conditioning? We leave these questions open and regard them as promising directions for future work.
\clearpage

\bibliographystyle{assets/plainnat}
\bibliography{references}
\clearpage
\beginappendix
\appendix

\section{Normalization as Hypersphere Projection}
\label{app:normalization-as-projection}

Throughout this appendix we work with idealized normalization maps: we omit the numerical stabilizer $\varepsilon$ in the denominators, and we set all additive bias terms to zero, since a translation does not affect the projection interpretation. For $r>0$, let
\[
\mathbb{S}^{d-1}_{r} := \bigl\{\mathbf{x}\in\mathbb{R}^{d}:\,
\lVert\mathbf{x}\rVert_{2}=r\bigr\},
\]
and, for $\mathbf{x}\in\mathbb{R}^{d}\setminus\{\boldsymbol{0}\}$, define the radial projection onto $\mathbb{S}^{d-1}_{r}$ by
\[
\Pi_{r}(\mathbf{x}):=r\cdot\frac{\mathbf{x}}{\lVert\mathbf{x}\rVert_{2}}.
\]
For every nonzero $\mathbf{x}$, the point $\Pi_{r}(\mathbf{x})$ is also the unique Euclidean metric projection of $\mathbf{x}$ onto $\mathbb{S}^{d-1}_{r}$. More generally, for a linear subspace $\mathcal{V}\subseteq\mathbb{R}^{d}$ we write
\[
\mathbb{S}_{r}(\mathcal{V}):=\bigl\{\mathbf{y}\in\mathcal{V}:\,
\lVert\mathbf{y}\rVert_{2}=r\bigr\},
\]
which is a $(\dim\mathcal{V}-1)$-dimensional sphere isometric to $\mathbb{S}^{\dim\mathcal{V}-1}_{r}$.

\paragraph{LayerNorm.}
Inspired by BatchNorm~\citep{pmlr-v37-ioffe15}, \citet{ba2016layer} proposed LayerNorm, which normalizes each input vector along its feature dimension. Let
$\mathbf{1}=(1,\ldots,1)^{\top}\in\mathbb{R}^{d}$ and define
\[
\mu=\frac{1}{d}\mathbf{1}^{\top}\mathbf{x},
\quad
\mathbf{P}=\mathbf{I}-\frac{1}{d}\mathbf{1}\mathbf{1}^{\top}.
\]
Since $\mathbf{P}=\mathbf{P}^{\top}=\mathbf{P}^{2}$ and $\ker\mathbf{P}=\operatorname{span}\{\mathbf{1}\}$, the matrix $\mathbf{P}$ is the orthogonal projection onto the mean-zero hyperplane
\[
\mathcal{H}=\bigl\{\mathbf{y}\in\mathbb{R}^{d}:\mathbf{1}^{\top}\mathbf{y}=0\bigr\},
\quad \dim\mathcal{H}=d-1,
\]
and the centered input satisfies $\mathbf{x}-\mu\mathbf{1}=\mathbf{P}\mathbf{x}$. Consequently, the coordinate-wise standard deviation of $\mathbf{x}$ satisfies
\[
\sigma=\sqrt{\frac{1}{d}\sum_{i=1}^{d}(x_{i}-\mu)^{2}}
=\frac{\lVert\mathbf{P}\mathbf{x}\rVert_{2}}{\sqrt{d}},
\]
so that $\sigma>0$ if and only if $\mathbf{P}\mathbf{x}\neq\boldsymbol{0}$, i.e., $\mathbf{x}\notin\operatorname{span}\{\mathbf{1}\}$; we restrict attention to
this case, on which LayerNorm is well defined. Let $\boldsymbol{\gamma}\in\mathbb{R}^{d}$ denote the learnable gain and let $\mathbf{D}_{\boldsymbol{\gamma}}=\operatorname{diag}(\boldsymbol{\gamma})$. LayerNorm can be written as
\begin{equation}\label{eq:layernorm-as-projection}
\begin{split}
\operatorname{LayerNorm}(\mathbf{x})
&=\boldsymbol{\gamma}\odot\frac{\mathbf{x}-\mu\mathbf{1}}{\sigma}\\
&=\mathbf{D}_{\boldsymbol{\gamma}}\left(\sqrt{d}\cdot
\frac{\mathbf{P}\mathbf{x}}{\lVert\mathbf{P}\mathbf{x}\rVert_{2}}\right)\\
&=\mathbf{D}_{\boldsymbol{\gamma}}\left(\Pi_{\sqrt{d}}(\mathbf{P}\mathbf{x})\right).
\end{split}
\end{equation}

Equation~\eqref{eq:layernorm-as-projection} shows that, before the coordinate-wise gain is applied, LayerNorm performs two successive projections: the orthogonal projection of $\mathbf{x}$ onto $\mathcal{H}$, followed by the radial projection of the centered vector onto $\mathbb{S}_{\sqrt{d}}(\mathcal{H})$, a $(d-2)$-dimensional sphere of radius $\sqrt{d}$ contained in $\mathcal{H}$. The image of the complete map is therefore
\[
\mathbf{D}_{\boldsymbol{\gamma}}\bigl(\mathbb{S}_{\sqrt{d}}(\mathcal{H})\bigr)
=\bigl\{\mathbf{D}_{\boldsymbol{\gamma}}\mathbf{y}:\,\mathbf{y}\in
\mathbb{S}_{\sqrt{d}}(\mathcal{H})\bigr\},
\]
which lies in the hyperplane $\mathbf{D}_{\boldsymbol{\gamma}}\mathcal{H}$ and is in general an ellipsoid, not necessarily axis-aligned, rather than a sphere. Assume $d \ge 3$ and $\gamma_{i}\neq 0$ for every $i$. The image is a round sphere if and only if $\mathbf{D}_{\boldsymbol{\gamma}}$ acts conformally on $\mathcal{H}$, i.e., $\mathbf{P}\mathbf{D}_{\boldsymbol{\gamma}}^{2}\mathbf{P}
= c^{2}\mathbf{P}$ for some $c>0$. Testing this identity on the pairs $\mathbf{e}_{i}-\mathbf{e}_{j}$ and $\mathbf{e}_{i}-\mathbf{e}_{k}$ with $i,j,k$ distinct, which is possible because $d\ge 3$, yields $\gamma_{i}^{2}=c^{2}$ for every $i$. Hence the image is a sphere precisely when $\lvert\gamma_{1}\rvert=\lvert\gamma_{2}\rvert=\cdots=\lvert\gamma_{d}\rvert$. Since $\boldsymbol{\gamma}$ is typically initialized at $\mathbf{1}$, at initialization LayerNorm reduces to $\mathbf{x}\mapsto\Pi_{\sqrt{d}}(\mathbf{P}\mathbf{x})$ and all outputs share the same norm $\sqrt{d}$. Once training drives $\boldsymbol{\gamma}$ away from $\mathbf{1}$, the outputs leave $\mathbb{S}_{\sqrt{d}}(\mathcal{H})$; they remain on some sphere only in the special case in which all $\lvert\gamma_{i}\rvert$ coincide, and are otherwise confined to the ellipsoid $\mathbf{D}_{\boldsymbol{\gamma}}\left(\mathbb{S}_{\sqrt{d}}(\mathcal{H})\right)$.

\paragraph{RMSNorm.}
RMSNorm~\citep{NEURIPS2019_1e8a1942} dispenses with the mean-centering step of LayerNorm and normalizes an input solely by its root mean square. For
$\mathbf{x}\in\mathbb{R}^{d}\setminus\{\boldsymbol{0}\}$, define the root mean square (RMS) operation as
\[
\operatorname{RMS}(\mathbf{x})=\sqrt{\frac{1}{d}\sum_{i=1}^{d}x_{i}^{2}}
=\frac{\lVert\mathbf{x}\rVert_{2}}{\sqrt{d}}.
\]
Let $\boldsymbol{\gamma}\in\mathbb{R}^{d}$ denote the learnable coordinate-wise gain and let $\mathbf{D}_{\boldsymbol{\gamma}}=\operatorname{diag}(\boldsymbol{\gamma})$. RMSNorm is then given by
\begin{equation}\label{eq:rmsnorm-as-projection}
\begin{split}
\operatorname{RMSNorm}(\mathbf{x})
&=\boldsymbol{\gamma}\odot\frac{\mathbf{x}}{\operatorname{RMS}(\mathbf{x})} \\
&=\mathbf{D}_{\boldsymbol{\gamma}}\left(\sqrt{d}\cdot
\frac{\mathbf{x}}{\lVert\mathbf{x}\rVert_{2}}\right)\\
&=\mathbf{D}_{\boldsymbol{\gamma}}\Pi_{\sqrt{d}}(\mathbf{x}).
\end{split}
\end{equation}
The normalization stage of RMSNorm, prior to the coordinate-wise gain, is therefore precisely the radial projection onto $\mathbb{S}^{d-1}_{\sqrt{d}}$. The complete map, however, is in general no longer spherical: its image is $\mathbf{D}_{\boldsymbol{\gamma}}\!\left(\mathbb{S}^{d-1}_{\sqrt{d}}\right)$, which, when $\gamma_{i}\neq 0$ for every $i$, is an axis-aligned ellipsoid whose semi-axes have lengths $\sqrt{d}\,\lvert\gamma_{i}\rvert$; if some $\gamma_{i}$ vanishes, the image degenerates to a lower-dimensional set. For $d\ge 2$, the
spherical geometry is retained precisely when all entries of $\boldsymbol{\gamma}$ share a common absolute value $c\neq 0$, in which case the image is the rescaled sphere $\mathbb{S}^{d-1}_{c\sqrt{d}}$. Since $\boldsymbol{\gamma}$ is typically initialized at $\mathbf{1}$, RMSNorm coincides with the projection onto $\mathbb{S}^{d-1}_{\sqrt{d}}$ at initialization; once training updates $\boldsymbol{\gamma}$, its outputs generally drift off the hypersphere.

\paragraph{ScaleNorm.}
Among commonly used normalization methods, ScaleNorm~\citep{nguyen-salazar-2019-transformers} is the one most directly related to hypersphere projection. For
$\mathbf{x}\in\mathbb{R}^{d}\setminus\{\boldsymbol{0}\}$ it is defined by
\begin{equation}
\operatorname{ScaleNorm}(\mathbf{x})=g\,\frac{\mathbf{x}}{\lVert\mathbf{x}\rVert_{2}},
\label{eq:scalenorm}
\end{equation}
where $g\in\mathbb{R}$ is a learnable scalar parameter. If $g>0$, then
\[
\operatorname{ScaleNorm}(\mathbf{x})=\Pi_{g}(\mathbf{x}),
\]
so ScaleNorm maps every nonzero input onto the hypersphere $\mathbb{S}^{d-1}_{g}$. More generally, when $g$ is unconstrained,
\[
\lVert\operatorname{ScaleNorm}(\mathbf{x})\rVert_{2}=\lvert g\rvert .
\]
Thus, for $g<0$ the spherical interpretation still holds: the map is the composition of the radial projection $\Pi_{\lvert g\rvert}$ onto $\mathbb{S}^{d-1}_{\lvert g\rvert}$ with the antipodal map $\mathbf{x}\mapsto-\mathbf{x}$, and only the degenerate case $g=0$ collapses the image to the origin. If an explicitly positive radius parameter is needed, one may use a positive reparameterization such as replacing $g$ with $\operatorname{Softplus}(g)$, although this is not required for the geometric interpretation.

\paragraph{Connecting ScaleNorm with Proj-SpheretNorm.}
The relationship between ScaleNorm and Proj-SpheretNorm can be made precise through their common radial-projection structure. Let $\mathbf{u}\in\mathbb{R}^{d}\setminus\{\boldsymbol{0}\}$ be the input to a generic Proj-SpheretNorm update, let $\mathcal{F}$ denote the associated sub-layer map, and let $\alpha\in\mathbb{R}$ be the residual coefficient. Set
\[
r=\lVert\mathbf{u}\rVert_{2},
\qquad
\mathbf{P}_{\mathbf{u}}=\mathbf{I}-\frac{\mathbf{u}\mathbf{u}^{\top}}{r^{2}},
\qquad
\mathbf{z}=\mathbf{P}_{\mathbf{u}}\mathcal{F}(\mathbf{u}).
\]
Here $\mathbf{P}_{\mathbf{u}}$ is the orthogonal projection onto $\mathbf{u}^{\perp}$, which is the tangent space of $\mathbb{S}^{d-1}_{r}$ at $\mathbf{u}$; in particular $\mathbf{u}^{\top}\mathbf{z}=0$ by construction. The Proj-SpheretNorm update can therefore be written as
\begin{equation}\label{eq:proj-spheretnorm-radial}
\mathbf{u}^{+}=r\,\frac{\mathbf{u}+\alpha\mathbf{z}}
{\lVert\mathbf{u}+\alpha\mathbf{z}\rVert_{2}}
=\Pi_{r}\bigl(\mathbf{u}+\alpha\mathbf{z}\bigr).
\end{equation}
Moreover, orthogonality gives
\[
\lVert\mathbf{u}+\alpha\mathbf{z}\rVert_{2}^{2}
=r^{2}+\alpha^{2}\lVert\mathbf{z}\rVert_{2}^{2}>0,
\]
so the radial projection is well defined whenever $\mathbf{u}\neq\boldsymbol{0}$. Equation~\eqref{eq:proj-spheretnorm-radial} also gives
$\lVert\mathbf{u}^{+}\rVert_{2}=\lVert\mathbf{u}\rVert_{2}$.

For $l=1,\,2,\,\ldots,\,L-1$, the $(l+1)$-th hidden state is given by
\[
\mathbf{h}_{l+1}=\operatorname{SpheretNorm}\!\left(\mathbf{h}_{l},\,
\mathcal{F}_{l}(\mathbf{h}_{l})\right).
\]
When Proj-SpheretNorm is used, the definition in Eq.~\eqref{eq:proj_spheretnorm} gives
\[
\mathbf{h}_{l+1}=\Pi_{\lVert\mathbf{h}_{l}\rVert_{2}}
\left(\mathbf{h}_{l}+\alpha_{l}\mathbf{z}_{l}\right),
\qquad
\mathbf{z}_{l}=\mathbf{P}_{\mathbf{h}_{l}}\mathcal{F}_{l}(\mathbf{h}_{l}),
\]
so that $\lVert\mathbf{h}_{l+1}\rVert_{2}=\lVert\mathbf{h}_{l}\rVert_{2}$ for $l=1,2,\ldots,L-1$. Consequently, the residual stream preserves the
hidden-state norm across layers, i.e., $\lVert\mathbf{h}_{1}\rVert_{2}=\lVert\mathbf{h}_{2}\rVert_{2}=\cdots=\lVert\mathbf{h}_{L}\rVert_{2}$.

At the beginning of the residual stream considered in our analysis, we take $\mathbf{h}_{1}$ to be the hidden state produced by the modified ScaleNorm, so that, by construction,
\[
\lVert\mathbf{h}_{1}\rVert_{2}=\operatorname{Softplus}(\gamma),
\]
which is strictly positive for every $\gamma\in\mathbb{R}$. We do not introduce notation for the representations preceding $\mathbf{h}_{1}$, and any operations applied after $\mathbf{h}_{L}$ lie outside the scope of this analysis.

Equation~\eqref{eq:proj-spheretnorm-radial} indicates that ScaleNorm and Proj-SpheretNorm share the same radial-projection operation but use it for different purposes. ScaleNorm fixes the radius of the initial hidden state $\mathbf{h}_{1}$, whereas each Proj-SpheretNorm sub-layer forms a tangent residual update and projects the resulting point back onto the sphere determined by the radius of its own input state. Under the above fixed-radius dynamics, every hidden state along the residual stream lies on the same hypersphere $\mathbb{S}^{d-1}_{\operatorname{Softplus}(\gamma)}$.

In summary, LayerNorm, RMSNorm, and ScaleNorm are all related to hypersphere projection, but in different ways. The normalization stages of LayerNorm and RMSNorm, taken before the gain, are spherical maps: LayerNorm first projects onto the zero-mean hyperplane and then performs a radial projection within that hyperplane, while RMSNorm performs a radial projection in the ambient space. Their complete maps, however, generally include a coordinate-wise gain and therefore need not have spherical images.

ScaleNorm differs in that it uses a single scalar gain. Its complete map therefore has a spherical image of radius $\lvert g\rvert$, and coincides exactly with the radial projection $\Pi_{g}$ when $g>0$. This positive-radius case provides the closest traditional counterpart to Proj-SpheretNorm. Specifically, Proj-SpheretNorm replaces the globally learned radius $g$ with the radius $\lVert\mathbf{h}_{l}\rVert_{2}$ carried by the current hidden state, and applies the radial projection to the point obtained after the tangent update. It can thus be viewed as a state-dependent, retraction-based variant of the radial projection used in ScaleNorm, designed to preserve the hidden-state norm exactly across the entire depth.

\section{Hyperparameters}\label{app:hyperparameters}
The shared hyperparameters used across all experiments are summarized in Table~\ref{tab:shared_hyperparameters}. The architecture hyperparameters for the small and medium model sizes are listed in Table~\ref{tab:arch_hyperparameters}. The training and validation losses in OpenWebText is illustrated in Figure~\ref{fig:loss_owt}.

\begin{table}[!ht]
\centering
\caption{Shared hyperparameters. The effective batch size is held fixed across all runs, independent of the number of GPUs.}
\label{tab:shared_hyperparameters}
\setlength{\tabcolsep}{4pt}
\begin{tabular}{l|c}
\toprule
Name & Value \\
\midrule
optimizer                     & AdamW \\
$\beta_{1}$                   & $0.9$ \\
$\beta_{2}$                   & $0.95$ \\
weight decay                  & $0.1$ \\
gradient clip                 & $1.0$ \\
dropout                       & $0.0$ \\
bias                          & enable \\
\# of iterations              & $50{,}000$ \\
\# of learning-rate decay iterations & $50{,}000$ \\
\# of warm-up iterations      & $2{,}000$ \\
block size (sequence length)  & $1{,}024$ \\
batch size (per GPU)          & $16$ \\
effective (global) batch size & $128$ sequences \\
tokens per update             & $131{,}072$ \\
total training tokens         & $\sim\!6.55$\,B \\
precision                     & \texttt{bfloat16} \\
torch.compile                 & disable \\
vocabulary size               & $50{,}304$ \\
\bottomrule
\end{tabular}
\end{table}

\begin{table}[!ht]
\centering
\caption{Architecture hyperparameters. Parameter counts include the tied token embedding and are reported for the Pre-LN baseline. All variants have identical parameter counts.}
\label{tab:arch_hyperparameters}
\setlength{\tabcolsep}{4pt}
\begin{tabular}{l|c|c|c}
\toprule
Name & S & M & L \\
\midrule
\# of layers 
    & $12$ 
    & $24$
    & $36$ \\
\# of heads 
    & $12$ 
    & $16$
    & $20$ \\
\# of params.
    & $\sim\! 0.12$\,B
    & $\sim\! 0.35$\,B
    & $\sim\! 0.77$\,B \\
hidden dim. 
    & $768$ 
    & $1024$ 
    & $1280$ \\
lr 
    & $6 \times 10^{-4}$ 
    & $3 \times 10^{-4}$ 
    & $2.5 \times 10^{-4}$ \\
min. lr 
    & $6 \times 10^{-5}$ 
    & $3 \times 10^{-5}$ 
    & $2.5 \times 10^{-5}$ \\
\bottomrule
\end{tabular}
\end{table}

\begin{figure}[!ht]
  \centering
  \begin{subfigure}{0.32\textwidth}
    \includegraphics[width=\linewidth]{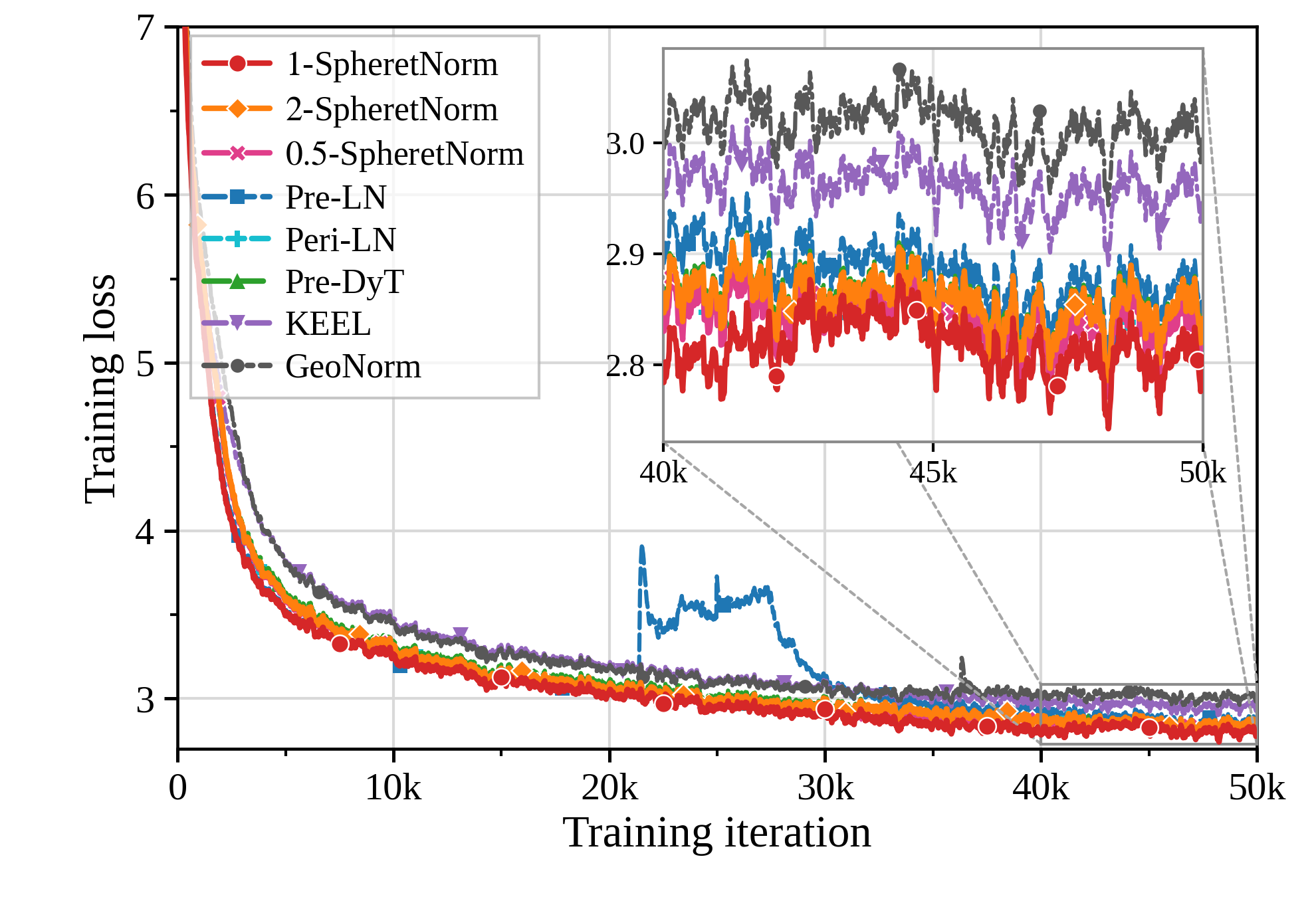}
    \caption{Training loss.}
  \end{subfigure}\hfill
  \begin{subfigure}{0.32\textwidth}
    \includegraphics[width=\linewidth]{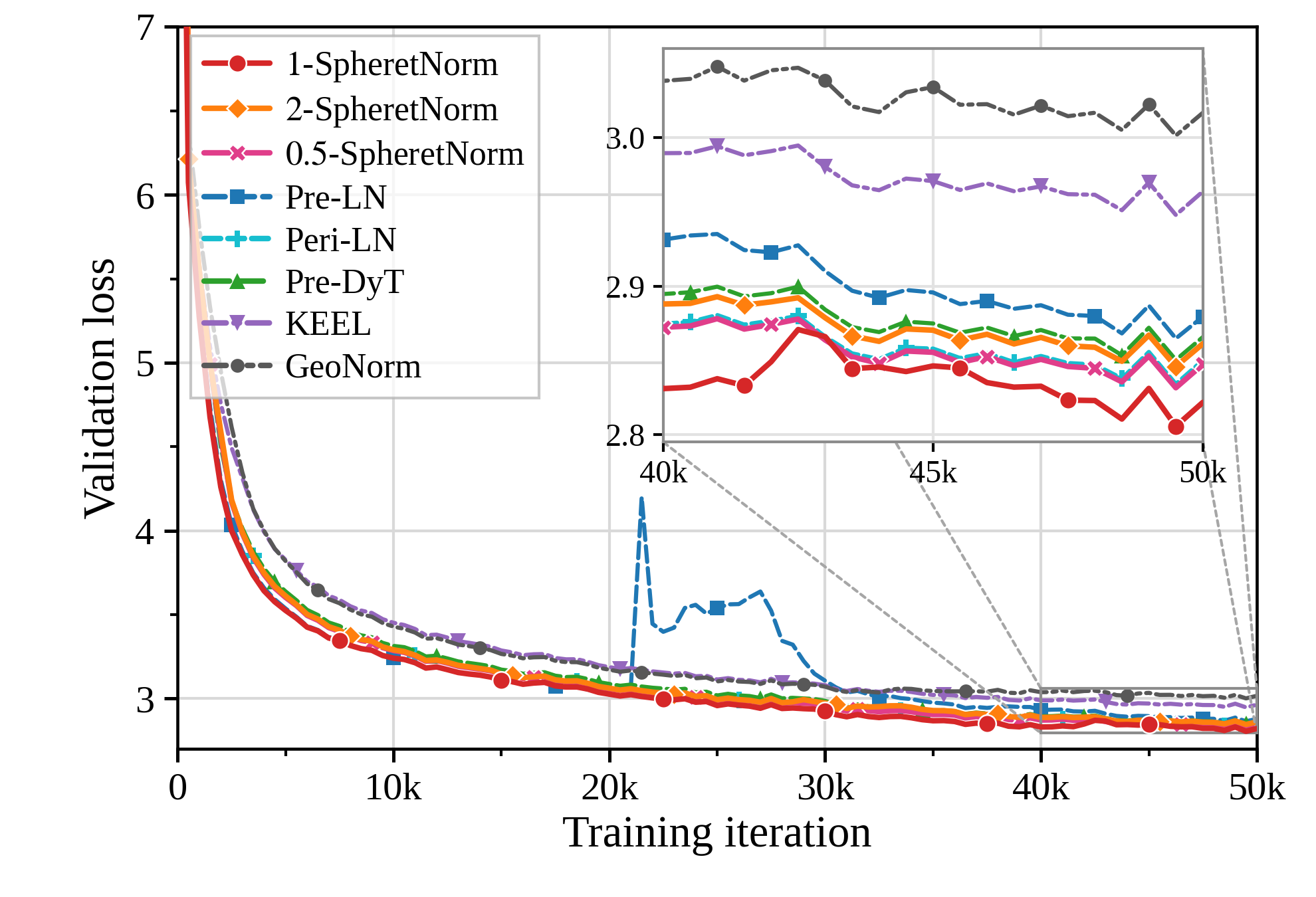}
    \caption{Validation loss.}
  \end{subfigure}
  \caption{Training and validation losses in the OpenWebText.}
  \label{fig:loss_owt}
\end{figure}

\section{Pseudocode for SpheretNorms}\label{app:pseudocode}
In this section, we provide PyTorch-style pseudocode for our SpheretNorms: Proj-SpheretNorm, Cay-SpheretNorm, and $p$-SpheretNorm. 

\begin{codeblock}[!ht]
\caption{PyTorch-style pseudocode for GeoNorm.}
\label{code:geonorm}
\centering
\begin{lstlisting}[style=pycode]
def geonorm(h_in, Fh, decay_method, layer_idx, num_layers, angle_upper_bound, alpha):
    h_in_norm_sq = (h_in * h_in).sum(dim=-1, keepdim=True).clamp(min=1e-12)
    h_in_norm = h_in_norm_sq.sqrt()
    h_in_trs_dot_Fh = (h_in * Fh).sum(dim=-1, keepdim=True)

    z = Fh - (h_in_trs_dot_Fh / h_in_norm_sq) * h_in
    z_norm_sq = (z * z).sum(dim=-1, keepdim=True).clamp(min=1e-12)
    z_norm = z_norm_sq.sqrt()
    z_hs = z / z_norm

    if decay_method == 'harmonic':
        alpha = torch.clamp(F.softplus(alpha), max=1.0) / layer_idx
    elif decay_method == 'linear':
        alpha = torch.clamp(F.softplus(alpha), max=1.0) * (num_layers - layer_idx) / num_layers
    elif decay_method == 'sqrt':
        alpha = torch.clamp(F.softplus(alpha), max=1.0) / math.sqrt(layer_idx)

    beta = (alpha * (z_norm / h_in_norm)).clamp(max=angle_upper_bound)
    h_out = torch.cos(beta) * h_in + h_in_norm * torch.sin(beta) * z_hs

    return h_out
\end{lstlisting}
\end{codeblock}

\begin{codeblock}[!ht]
\caption{PyTorch-style pseudocode for Proj-SpheretNorm.}
\label{code:proj_spheretnorm}
\centering
\begin{lstlisting}[style=pycode]
def proj_spheretnorm(h_in, Fh, layer_idx, alpha):
    h_in_norm_sq = (h_in * h_in).sum(dim=-1, keepdim=True).clamp(min=1e-12)
    h_in_norm = h_in_norm_sq.sqrt()
    h_in_trs_dot_Fh = (h_in * Fh).sum(dim=-1, keepdim=True)

    z = Fh - (h_in_trs_dot_Fh / h_in_norm_sq) * h_in

    alpha = torch.clamp(F.softplus(alpha), max=1.0) / math.sqrt(self.layer_idx)

    h_plus = h_in + alpha * z
    h_plus_norm_sq = (h_plus * h_plus).sum(dim=-1, keepdim=True).clamp(min=1e-12)
    h_plus_norm = h_plus_norm_sq.sqrt()
    h_out = h_in_norm * (h_plus / h_plus_norm)

    return h_out
\end{lstlisting}
\end{codeblock}

\begin{codeblock}[!ht]
\caption{PyTorch-style pseudocode for Cay-SpheretNorm.}
\label{code:cay_spheretnorm}
\centering
\begin{lstlisting}[style=pycode]
def cay_spheretnorm(h_in, Fh, layer_idx, angle_upper_bound, alpha):
    h_in_norm_sq = (h_in * h_in).sum(dim=-1, keepdim=True).clamp(min=1e-12)
    h_in_norm = h_in_norm_sq.sqrt()
    h_in_trs_dot_Fh = (h_in * Fh).sum(dim=-1, keepdim=True)

    z = Fh - (h_in_trs_dot_Fh / h_in_norm_sq) * h_in
    z_norm_sq = (z * z).sum(dim=-1, keepdim=True).clamp(min=1e-12)
    z_norm = z_norm_sq.sqrt()

    alpha = torch.clamp(F.softplus(alpha), max=1.0) / math.sqrt(layer_idx)

    beta_max = math.tan(angle_upper_bound * 0.5) * 2.0
    gamma = torch.minimum(alpha, beta_max * (h_in_norm / z_norm))
    beta = gamma * (z_norm / h_in_norm)
    beta_sq = beta * beta
    h_out = ((4 - beta_sq) * h_in + 4 * gamma * z) / (4 + beta_sq)

    return h_out
\end{lstlisting}
\end{codeblock}

\begin{codeblock}[!ht]
\caption{PyTorch-style pseudocode for $p$-SpheretNorm.}
\label{code:p_spheretnorm}
\centering
\begin{lstlisting}[style=pycode]
def p_spheretnorm(h_in, Fh, layer_idx, angle_upper_bound, alpha, p):
    h_in_norm_sq = (h_in * h_in).sum(dim=-1, keepdim=True).clamp(min=1e-12)
    h_in_norm = h_in_norm_sq.sqrt()
    h_in_trs_dot_Fh = (h_in * Fh).sum(dim=-1, keepdim=True)

    z = Fh - (h_in_trs_dot_Fh / h_in_norm_sq) * h_in
    z_norm_sq = (z * z).sum(dim=-1, keepdim=True).clamp(min=1e-12)
    z_norm = z_norm_sq.sqrt()
    z_hs = z / z_norm

    alpha = torch.clamp(F.softplus(alpha), max=1.0) / math.sqrt(self.layer_idx)

    beta = alpha * (z_norm / h_in_norm)
    theta = p * torch.arctan(beta / p)
    if p > 1.0:
        theta = torch.clamp(theta, max=angle_upper_bound)

    h_out = torch.cos(theta) * h_in + (h_in_norm * torch.sin(theta)) * z_hs

    return h_out
\end{lstlisting}
\end{codeblock}

\section{Technical Proofs for Section~\ref{sec:Proposed Method}}
\label{app:analysis-proofs}

\subsection{Proof of
Proposition~\ref{prop:cayley-sphere}}
\label{app:cayley-sphere}

Let $\mathcal{U}=\operatorname{span}\{\mathbf{h}, \mathbf{v}\}$. We first prove uniqueness. Suppose that $W\in\mathfrak{so}(d)$ satisfies
\[
\mathbf{W}\mathbf{h}=\mathbf{v},\qquad\operatorname{range}(\mathbf{W})\subseteq\mathcal{U}.
\]

For every $\mathbf{x}\in\mathcal{U}^{\perp}$ and $\mathbf{y}\in\mathcal{U}$, skew-symmetry gives
\[
\langle \mathbf{W}\mathbf{x}, \mathbf{y}\rangle=-\langle \mathbf{x}, \mathbf{W}\mathbf{y}\rangle.
\]
Since $\mathbf{W}\mathbf{y}\in\mathcal{U}$ and $\mathbf{x}\in\mathcal{U}^{\perp}$, it follows that $\langle \mathbf{W}\mathbf{x},\mathbf{y}\rangle=0$ for every $\mathbf{y}\in\mathcal{U}$. On the other hand,
$\mathbf{W}\mathbf{x}\in\mathcal{U}$ by the range condition. Therefore,
\[
\mathbf{W}\mathbf{x}=\boldsymbol{0},\qquad \mathbf{x}\in\mathcal{U}^{\perp}.
\]
Thus, $\mathbf{W}$ is completely determined by its restriction to
$\mathcal{U}$.

If $\mathbf{v}=\boldsymbol{0}$, then $\mathcal{U}=\operatorname{span}\{\mathbf{h}\}$ and $\mathbf{W}\mathbf{h}=\boldsymbol{0}$. Hence $\mathbf{W}=\boldsymbol{0}$, which agrees with the claimed formula.

Suppose now that $\mathbf{v} \neq \boldsymbol{0}$, and define the orthonormal vectors
\[
\mathbf{e}_1=\frac{\mathbf{h}}{r},\qquad \mathbf{e}_2=\frac{\mathbf{v}}{\|\mathbf{v}\|_2}.
\]
Since $\mathbf{W}\mathbf{h}=\mathbf{v}$, we have
\[
\mathbf{W}\mathbf{e}_1=\frac{\|\mathbf{v}\|_2}{r}\mathbf{e}_2.
\]
Skew-symmetry then uniquely determines the action on $\mathbf{e}_2$: $\mathbf{W}\mathbf{e}_2=-\frac{\|\mathbf{v}\|_2}{r}\mathbf{e}_1$. Together with $\mathbf{W}=\boldsymbol{0}$ on $\mathcal{U}^{\perp}$, this proves that at most one such skew-symmetric matrix exists.

Now define
\[
\widehat{\mathbf{W}}=\frac{\mathbf{v}\mathbf{h}^\top-\mathbf{h}\mathbf{v}^\top}{r^2}.
\]
Then $\widehat{\mathbf{W}}^{\top}=-\widehat{\mathbf{W}}$, so $\widehat{\mathbf{W}}\in\mathfrak{so}(d)$. Since
$\mathbf{h}^\top \mathbf{v}=0$ and $\|\mathbf{h}\|_2^2=r^2$,
\[
\widehat{\mathbf{W}}\mathbf{h}=\frac{\mathbf{v} \mathbf{h}^\top \mathbf{h}-\mathbf{h} \mathbf{v}^\top \mathbf{h}}{r^2}=\mathbf{v}.
\]
Moreover, $\operatorname{range}(\widehat{\mathbf{W}})\subseteq\operatorname{span}\{\mathbf{h},\mathbf{v}\}$. Therefore, by uniqueness,
\[
\mathbf{W}=\widehat{\mathbf{W}}=\frac{\mathbf{v}\mathbf{h}^\top-\mathbf{h}\mathbf{v}^\top}{r^2}.
\]

We next derive the closed form of the Cayley update. Define
$\rho=\frac{\|\mathbf{v}\|_2}{r}$, we obtain
\[
\mathbf{W}\mathbf{v}=\frac{\mathbf{v} \mathbf{h}^\top \mathbf{v}-\mathbf{h} \mathbf{v}^\top \mathbf{v}}{r^2}=-\rho^2\mathbf{h}.
\]

The matrix $\mathbf{I} - \frac{1}{2}\mathbf{W}$ is invertible. Indeed, if $\left(\mathbf{I}-\frac{1}{2}\mathbf{W}\right)\mathbf{x}=0$, then
$\mathbf{x}=\frac{1}{2}\mathbf{W}\mathbf{x}$. Taking the inner product with $\mathbf{x}$ and using $\mathbf{x}^\top \mathbf{W}\mathbf{x}=\boldsymbol{0}$ gives $\|\mathbf{x}\|_2^2=0$, and hence $\mathbf{x}=\boldsymbol{0}$.

Set $a=\frac{4-\rho^2}{4+\rho^2}$ and $b=\frac{4}{4+\rho^2}$. Using $\mathbf{W}\mathbf{h}=\mathbf{v}$ and $\mathbf{W}\mathbf{v}=-\rho^2\mathbf{h}$, we have
\begin{align*}
\left(\mathbf{I}-\frac{1}{2}\mathbf{W}\right)(a\mathbf{h}+b\mathbf{v})
&=a\mathbf{h}+b\mathbf{v}-\frac{a}{2}\mathbf{v}+\frac{b\rho^2}{2}\mathbf{h}\\
&=\left(a+\frac{b\rho^2}{2}\right)\mathbf{h}+\left(b-\frac{a}{2}\right)\mathbf{v}.
\end{align*}
The definitions of $a$ and $b$ give $a+\frac{b\rho^2}{2}=1$ and $b-\frac{a}{2}=\frac{1}{2}$. Therefore,
\[
\left(\mathbf{I}-\frac{1}{2}\mathbf{W}\right)(a\mathbf{h}+b\mathbf{v})=\mathbf{h}+\frac{1}{2}\mathbf{v}
=\left(\mathbf{I}+\frac{1}{2}\mathbf{W}\right)\mathbf{h}.
\]
Since $\mathbf{I}-\frac{1}{2}\mathbf{W}$ is invertible,
\[
\begin{aligned}
\mathcal{C}(W)\mathbf{h}
&=\left(\mathbf{I}-\frac{1}{2}\mathbf{W}\right)^{-1}\left(\mathbf{I}+\frac{1}{2}\mathbf{W}\right)\mathbf{h}\\
&=a\mathbf{h}+b\mathbf{v}\\
&=\frac{4-\rho^2}{4+\rho^2}\mathbf{h}+\frac{4}{4+\rho^2}\mathbf{v}.
\end{aligned}
\]
This proves the claimed closed-form expression.

\subsection{Proof of Proposition~\ref{prop:p-special-cases}}
\label{app:{prop:p-special-cases}}
For $p=1$, we have $\theta_l=\arctan(\beta_l)$, and hence
\[
    \cos\!\left(\theta_l\right)=\frac{1}{\sqrt{1+\beta_l^2}},
    \qquad\sin\!\left(\theta_l\right)=
    \frac{\beta_l}{\sqrt{1+\beta_l^2}}.
\]
Since $\beta_l=\alpha_ls_l/r_l$ and
$\mathbf{h}_{l}^\top\mathbf{z}_l=0$, it follows that
\[
    \mathbf{h}_{l+1}^{(1)}=\frac{\mathbf{h}_{l}+\alpha_l\mathbf{z}_l}{\sqrt{1+\beta_l^2}}
    =r_l\frac{\mathbf{h}_{l}+\alpha_l\mathbf{z}_l}{\|\mathbf{h}_{l}+\alpha_l\mathbf{z}_l\|_2}.
\]
This is precisely the Proj-SpheretNorm update.

For $p=2$, the double-angle identities give
\[
\cos\!\left(2\arctan\frac{\beta_l}{2}\right)=
    \frac{4-\beta_l^2}{4+\beta_l^2},\,
    \sin\!\left(2\arctan\frac{\beta_l}{2}\right)=
    \frac{4\beta_l}{4+\beta_l^2}.
\]
Substituting $\beta_l=\alpha_ls_l/r_l$ into
Definition~\ref{def:p-spheretnorm} yields
\[
    \mathbf{h}_{l+1}^{(2)}=
    \frac{4-\beta_l^2}{4+\beta_l^2}\mathbf{h}_{l}
    +\frac{4\alpha_l}{4+\beta_l^2}\mathbf{z}_l,
\]
which is the Cay-SpheretNorm update.

For every fixed $\beta_l$,
\[
    \lim_{p\to+\infty}p\arctan\left(\frac{\beta_l}{p}\right)
    =\beta_l.
\]
The continuity of the sine and cosine functions therefore
gives the exponential-map limit.

Finally,
\[
    \left|p\arctan\left(\frac{\beta_l}{p}\right)\right|
    \leq\frac{\pi p}{2}\to 0\qquad\text{as }p\to 0^+.
\]
Thus $\cos(\theta_l)\to 1$ and $\sin(\theta_l)\to 0$, proving convergence to the identity mapping.

\subsection{Proof of Proposition~\ref{prop:fixed-radius-dynamics}}
\label{app:fixed-radius-dynamics}

Consistent with the $L$ retained hidden states
$\mathbf{h}_1,\ldots, \mathbf{h}_L$, we index the individual
SpheretNorm updates by $l=1,\ldots,L-1$. Let
\[
\mathbf{h}_{l+1}=\operatorname{SpheretNorm}
\bigl(\mathbf{h}_l,\mathcal{F}_l(\mathbf{h}_l)\bigr),
\]
where the step-size and retraction parameters are suppressed
from the notation. Define $r_l=\|\mathbf{h}_l\|_2$, $\mathbf{P}_l
=\mathbf{I}-\frac{\mathbf{h}_l\mathbf{h}_l^\top}{r_l^2}$, and $\mathbf{z}_l=\mathbf{P}_l\mathcal{F}_l(\mathbf{h}_l)$. Then $\mathbf{h}_l^\top\mathbf{z}_l=0$. Suppose first that $\mathbf{z}_l\neq 0$. For an appropriate intrinsic angle $\theta_l$, every SpheretNorm variant considered in this work admits the representation
\[
\mathbf{h}_{l+1}=\cos(\theta_l)\mathbf{h}_l
+r_l\sin(\theta_l)\frac{\mathbf{z}_l}{\|\mathbf{z}_l\|_2}.
\]
Since $\mathbf{h}_l^\top\mathbf{z}_l=0$, the two terms on
the right-hand side are orthogonal, and therefore
\[
\begin{aligned}
\|\mathbf{h}_{l+1}\|_2^2&=\cos^2(\theta_l)\|\mathbf{h}_l\|_2^2
+r_l^2\sin^2(\theta_l)\left\|\frac{\mathbf{z}_l}{\|\mathbf{z}_l\|_2}\right\|_2^2 =r_l^2.
\end{aligned}
\]
Hence, $\|\mathbf{h}_{l+1}\|_2=\|\mathbf{h}_l\|_2$. When $\mathbf{z}_l=0$, the continuously extended update
satisfies $\mathbf{h}_{l+1}=\mathbf{h}_l$, and the same
identity holds. Therefore, by induction,
\[
\|\mathbf{h}_l\|_2=\|\mathbf{h}_1\|_2,\qquad l=1,\ldots,L.
\]

Since $\mathbf{h}_1$ is the output of the modified
ScaleNorm operation,
\[
\|\mathbf{h}_1\|_2=\operatorname{Softplus}(\gamma).
\]
Consequently,
\[
\|\mathbf{h}_l\|_2=\operatorname{Softplus}(\gamma),\qquad l=1,\ldots,L.
\]
In particular,
\[
\|\mathbf{h}_L\|_2=\operatorname{Softplus}(\gamma),
\]
which is stronger than, and therefore implies, the stated
upper bound
\[
\|\mathbf h_L\|_2\leq\operatorname{Softplus}(\gamma).
\]

\subsection{Second-Order Retraction and Local Consistency}
\label{app:p-angular-properties}

We prove the smoothness, second-order retraction, and local-consistency properties used in Section~1.3. Fix $p>0$, let $\mathbf{h}\in\mathbb{S}_r^{d-1}$, $\mathbf{v}\in T_{\mathbf{h}}\mathbb{S}_r^{d-1}$, and define $\rho=\frac{\|\mathbf{v}\|_2}{r}$, $\theta_p(\rho)=p\arctan\left(\frac{\rho}{p}\right)$.

For $\rho>0$, define
\[
a_p(\rho)=\cos\bigl(\theta_p(\rho)\bigr),\qquad
b_p(\rho)=\frac{\sin\bigl(\theta_p(\rho)\bigr)}{\rho},
\]
and set $b_p(0)=1$. The $p$-angular map can then be
written as
\[
\mathcal{R}_{\mathbf{h}}^{(p)}(\mathbf{v})=a_p(\rho)\mathbf{h}
+b_p(\rho)\mathbf{v}.
\]

\paragraph{Smoothness at the zero section.}
The apparent singularity in $b_p(\rho)$ at $\rho=0$
is removable. To see this, consider the real-analytic odd
function
\[
\widetilde{\theta}_p(t)=p\arctan\left(\frac{t}{p}\right),
\qquad t\in\mathbb{R}.
\]
Define $A_p(t)=\cos\bigl(\widetilde{\theta}_p(t)\bigr)$ and
$B_p(t)=
\begin{cases}
\displaystyle
\frac{\sin\bigl(\widetilde{\theta}_p(t)\bigr)}{t},& t\neq0,\\[1.2ex]
1,& t=0.
\end{cases}$. Since $\widetilde{\theta}_p$ is odd, both $A_p$ and
$B_p$ are even real-analytic functions in a neighborhood of the origin. Hence, there exist real-analytic functions $\widehat A_p$ and $\widehat B_p$, defined near zero, such that
\[
A_p(t)=\widehat A_p(t^2),\qquad B_p(t)=\widehat B_p(t^2).
\]
Consequently,
\[
\mathcal{R}_{\mathbf{h}}^{(p)}(\mathbf{v})=\widehat A_p
\left(\frac{\|\mathbf{v}\|_2^2}{r^2}\right)\mathbf{h}+\widehat B_p\left(\frac{\|\mathbf{v}\|_2^2}{r^2}\right)\mathbf{v}.
\]
This representation is smooth at $\mathbf{v}=\mathbf0$.
Since $r>0$ is fixed on $\mathbb{S}_r^{d-1}$, it also
shows that $(\mathbf{h},\mathbf{v})\to \mathcal{R}_{\mathbf{h}}^{(p)}(\mathbf{v})$ is smooth on the tangent bundle $T\mathbb{S}_r^{d-1}$ near the zero section. Away from $\mathbf{v}=\mathbf0$, smoothness follows directly from the defining formula.

\paragraph{Sphere-valued property.}
Since $\mathbf{h}^\top\mathbf{v}=0$, $\|\mathbf{h}\|_2=r$, and $\|\mathbf{v}\|_2=r\rho$, we obtain, for $\mathbf{v}\neq\boldsymbol{0}$,
\begin{align*}
\left\|\mathcal{R}_{\mathbf{h}}^{(p)}(\mathbf{v})\right\|_2^2&=a_p(\rho)^2\|\mathbf{h}\|_2^2
+b_p(\rho)^2\|\mathbf{v}\|_2^2\\
&=r^2\cos^2\bigl(\theta_p(\rho)\bigr)+r^2\sin^2\bigl(\theta_p(\rho)\bigr)=r^2.
\end{align*}
The same conclusion is immediate at
$\mathbf{v}=\mathbf0$. Therefore, $\mathcal{R}_{\mathbf{h}}^{(p)}(\mathbf{v})\in\mathbb{S}_r^{d-1}$.

\paragraph{Retraction and second-order properties.}
As $\rho\to 0$,
\[
\theta_p(\rho)=\rho-\frac{\rho^3}{3p^2}+\mathcal{O}(\rho^5).
\]
It follows that
\[
a_p(\rho)=1-\frac{\rho^2}{2}+\left(\frac{1}{24}+\frac{1}{3p^2}
\right)\rho^4+\mathcal{O}(\rho^6),
\]
and
\[
b_p(\rho)=1-\left(\frac{1}{6}+\frac{1}{3p^2}\right)\rho^2+\mathcal{O}(\rho^4).
\]
Equivalently, as $s\to 0$, it follows that
\[
\widehat A_p(s)=1-\frac{s}{2}+\left(\frac{1}{24}+\frac{1}{3p^2}\right)s^2+\mathcal{O}(s^3),
\]
and
\[
\widehat B_p(s)=1-\left(\frac{1}{6}+\frac{1}{3p^2}\right)s+\mathcal{O}(s^2).
\]

Fix $\mathbf{v}\in T_{\mathbf{h}}\mathbb{S}_r^{d-1}$, set $\rho_0=\frac{\|\mathbf{v}\|_2}{r}$, and consider the curve $\mathbf{c}(t)=\mathcal{R}_{\mathbf{h}}^{(p)}(t\mathbf{v})$. Using the even-function representation above, we have
\[
\mathbf{c}(t)=\widehat A_p(t^2\rho_0^2)\mathbf{h}+t\widehat B_p(t^2\rho_0^2)\mathbf{v}.
\]
Therefore,
\[
\mathbf{c}(t)
=\mathbf{h}+t\mathbf{v}-\frac{t^2\|\mathbf{v}\|_2^2}{2r^2}\mathbf{h}+\mathcal{O}(t^3).
\]
Hence, $\mathbf{c}(0)=\mathbf{h}$, $\mathbf{c}'(0)=\mathbf{v}$ and
$\mathbf{c}''(0)=-\frac{\|\mathbf{v}\|_2^2}{r^2}\mathbf{h}$. The first two identities verify the defining conditions of a
retraction. Moreover, since
\[
N_{\mathbf{h}}\mathbb S_r^{d-1}=\operatorname{span}\{\mathbf{h}\},
\]
the acceleration $\mathbf{c}''(0)$ is normal to the
hypersphere. Equivalently,
\[
\mathbf{P}_{\mathbf{h}}\mathbf{c}''(0)=\boldsymbol{0},
\qquad\mathbf{P}_{\mathbf{h}}=\mathbf{I}-\frac{\mathbf{h}\mathbf{h}^\top}{r^2}.
\]
Hence, the retraction curve has no tangential acceleration at
the origin, which is precisely the second-order retraction
condition. Therefore, $\mathcal{R}^{(p)}$ is a second-order retraction on
$\mathbb{S}_r^{d-1}$ for every fixed $p>0$.

\paragraph{Local expansion of the SpheretNorm update.}
For the $l$-th SpheretNorm update, let $r_l=\|\mathbf{h}_l\|_2$, $\mathbf{P}_l=\mathbf{I}-\frac{\mathbf{h}_l\mathbf{h}_l^\top}{r_l^2}$, $\mathbf{z}_l=\mathbf{P}_l\mathcal{F}_l(\mathbf{h}_l)$, and
$s_l=\|\mathbf{z}_l\|_2$. Applying the preceding expansion to $\mathbf{v}=\alpha_l\mathbf{z}_l$ gives
\begin{align*}
\mathbf{h}_{l+1}={}&\mathbf{h}_l+\alpha_l\mathbf{z}_l-\frac{\alpha_l^2\|\mathbf{z}_l\|_2^2}{2r_l^2}\mathbf{h}_l\\
&-\left(\frac{1}{6}+\frac{1}{3p^2}\right)\frac{\alpha_l^3\|\mathbf{z}_l\|_2^2}{r_l^2}
\mathbf{z}_l+\mathcal{O}(\alpha_l^4).
\end{align*}
Here, the expansion is taken as $\alpha_l\to 0$, with $\mathbf{h}_l$ and
$\mathbf{z}_l$ fixed. More generally, the remainder is uniform over any local regime in which $r_l$ is bounded away from zero and $\eta_l=\frac{\|\mathbf{z}_l\|_2}{r_l}$ remains bounded. The linear term $\alpha_l\mathbf{z}_l$ is tangent to the hypersphere, while the quadratic term $-\frac{\alpha_l^2\|\mathbf{z}_l\|_2^2}{2r_l^2}\mathbf{h}_l$ is normal to the hypersphere at $\mathbf{h}_l$. Moreover, the parameter $p$ first appears in the cubic tangent correction
\[
-\left(\frac{1}{6}+\frac{1}{3p^2}\right)\frac{\alpha_l^3\|\mathbf{z}_l\|_2^2}{r_l^2}
\mathbf{z}_l.
\]
Thus, all members of the $p$-angular family have the same
first-order tangent dynamics and the same second-order
normal correction, while their finite-step dependence on
$p$ first appears at cubic order. Finally, since $\mathbf{z}_l=\mathbf{P}_l\mathcal{F}_l(\mathbf{h}_l)$, we obtain the first-order consistency relation
\[
\mathbf{h}_{l+1}=\mathbf{h}_l+\alpha_l\mathbf{P}_l\mathcal{F}_l(\mathbf{h}_l)
+\mathcal{O}(\alpha_l^2).
\]

\subsection{Depth-Wise Conditioning Criterion}
\label{app:depthwise-conditioning}

We first formulate the conditioning criterion for a generic
ordered sequence of individual SpheretNorm updates. Let
$M$ denote the number of such updates, and let
\[
\mathbf{J}_j=\frac{\partial \mathbf x_{j+1}}{\partial \mathbf x_j},
\qquad j=1,\ldots,M.
\]
For each $j$, define
$
q_j=\begin{cases}
|\cos(\theta_j)|,& s_j>0,\\
1,& s_j=0,
\end{cases}$ and $m_j=\max\{0,q_j-\Delta_j\}$. Assume that there exists a constant $m_*>0$, independent of $j$ and $M$, such that $m_j\geq m_*$, $j=1,\ldots,M$. Since $m_j\leq1$, we have
\[
-\log m_j=\int_{m_j}^{1}\frac{dt}{t}\leq\frac{1-m_j}{m_*}.
\]
Because $m_j>0$, the maximum in its definition is inactive,
and hence $m_j=q_j-\Delta_j$. Therefore,
\[
1-m_j=1-q_j+\Delta_j.
\]
For every real $\theta$, $1-|\cos\theta|\leq 1-\cos\theta\leq\frac{\theta^2}{2}$. Consequently, $1-m_j\leq\frac{|\theta_j|^2}{2}+\Delta_j$. It follows that
\begin{align*}
-\log\left(\prod_{j=1}^{M}m_j\right)&=-\sum_{j=1}^{M}\log m_j\\
&\leq\frac{1}{m_*}\left[\frac12\sum_{j=1}^{M}|\theta_j|^2+\sum_{j=1}^{M}\Delta_j\right].
\end{align*}
Hence,
\[
\prod_{j=1}^{M}m_j\geq\exp\left\{-\frac{1}{m_*}\left[\frac12\sum_{j=1}^{M}|\theta_j|^2
+\sum_{j=1}^{M}\Delta_j\right]\right\}.
\]

Let $\mathbf{G}_M=\mathbf{J}_M\mathbf{J}_{M-1}\cdots\mathbf{J}_1$. Combining the preceding estimate with the layerwise bounds of Proposition~\ref{prop:jacobian-bounds} gives
\[
\sigma_{\max}(\mathbf{G}_M)\leq\exp\left(\sum_{j=1}^{M}\Delta_j\right)
\]
and
\[
\sigma_{\min}(\mathbf{G}_M)\geq\exp\left\{-\frac{1}{m_*}\left[\frac12\sum_{j=1}^{M}|\theta_j|^2
+\sum_{j=1}^{M}\Delta_j\right]\right\}.
\]
Therefore, if $\sum_{j=1}^{M}\Delta_j=\mathcal{O}(1)$ and $\sum_{j=1}^{M}|\theta_j|^2=\mathcal{O}(1)$, then both
$\sigma_{\max}(\mathbf{G}_M)$ and $\sigma_{\min}(\mathbf{G}_M)^{-1}$
remain bounded independently of the number $M$ of individual SpheretNorm updates.

For the residual stream
$\mathbf{h}_1,\mathbf{h}_2,\ldots,\mathbf{h}_L$
considered in Appendix~\ref{app:normalization-as-projection} and Appendix~\ref{app:fixed-radius-dynamics}, there are $M=L-1$ individual SpheretNorm updates. More precisely, setting
\[
\mathbf{x}_j=\mathbf{h}_j,\qquad\mathbf{J}_j
=\frac{\partial\mathbf{h}_{j+1}}{\partial\mathbf{h}_j},\qquad j=1,\ldots,L-1,
\]
gives
\[
\mathbf{G}_M=\mathbf{G}_{L-1}=\frac{\partial\mathbf{h}_L}
{\partial\mathbf{h}_1}=\mathbf{J}_{L-1}\mathbf{J}_{L-2}\cdots\mathbf{J}_1.
\]
Therefore, the preceding depth-wise conditioning criterion
applies to the actual residual stream with
\[
\sum_{j=1}^{L-1}\Delta_j=\mathcal{O}(1),\qquad\sum_{j=1}^{L-1}|\theta_j|^2=\mathcal{O}(1).
\]
Proposition~\ref{prop:jacobian-bounds} uses $L$ as the generic number of individual SpheretNorm updates. When the result is applied
to the retained-state indexing $\mathbf{h}_1,\ldots,\mathbf{h}_L$, used in Appendix~\ref{app:normalization-as-projection} and Appendix~\ref{app:fixed-radius-dynamics}, the corresponding update count is $M=L-1$.

\subsection{Proof of the Depth-Schedule Growth Proposition}
\label{app:depth-schedule-growth}
The three schedules compared in Proposition~\ref{prop:depth-schedule-growth} are $\lambda_l=1$, $\lambda_l=l^{-1/2}$, and $\lambda_l=l^{-1}$. In particular, the third schedule is the harmonic schedule
$\lambda_l=l^{-1}$, for which
\[
\sum_{l=1}^L\lambda_l=\mathcal{O}(\log L),
\qquad
\sum_{l=1}^L\lambda_l^2=\mathcal{O}(1).
\]
By definition, $S_L=\sum_{l=1}^{L}\lambda_l a_l\kappa_l$. Since $0<a_l\leq1$ and $\kappa_l\leq K$, it follows that $S_L\leq K\sum_{l=1}^{L}\lambda_l$. Furthermore, since
\[
\left|p\arctan\left(\frac{\beta}{p}\right)\right|\leq |\beta|,
\qquad \beta\in\mathbb{R},
\]
we have $\left|\theta_l\right|\leq|\beta_l|=\lambda_l a_l\eta_l$. Using $a_l\eta_l\leq H$, we obtain
\[
A_L=\sum_{l=1}^{L}\left|\theta_l\right|^2
\leq\sum_{l=1}^{L}\lambda_l^2(a_l\eta_l)^2\leq H^2\sum_{l=1}^{L}\lambda_l^2.
\]

For the constant schedule $\lambda_l=1$, we have $\sum_{l=1}^{L}\lambda_l=L$ and $\sum_{l=1}^{L}\lambda_l^2=L$. For the layerwise square-root schedule
$\lambda_l=l^{-1/2}$,
\[
\sum_{l=1}^{L}\lambda_l=\sum_{l=1}^{L}\frac{1}{\sqrt l}
\leq 2\sqrt L=\mathcal{O}(\sqrt L),
\]
while
\[
\sum_{l=1}^{L}\lambda_l^2=\sum_{l=1}^{L}\frac{1}{l}\leq 1+\log L
=\mathcal{O}(\log L).
\]

For the harmonic schedule $\lambda_l=l^{-1}$,
\begin{align*}
\sum_{l=1}^{L}\lambda_l &=\sum_{l=1}^{L}\frac{1}{l}
=\mathcal{O}(\log L)
\,\mbox{ and}\\
\sum_{l=1}^{L}\lambda_l^2 &=\sum_{l=1}^{L}\frac{1}{l^2}\leq\frac{\pi^2}{6}
=\mathcal{O}(1).
\end{align*}
The stated growth rates follow.

\paragraph{Application to the retained-state indexing.}
Proposition~\ref{prop:depth-schedule-growth} uses \(L\) as the generic number of individual SpheretNorm updates. For the retained-state sequence $\mathbf{h}_1,\mathbf{h}_2,\ldots, \mathbf{h}_L$, the corresponding number of updates is $M=L-1$. For $l=1,\ldots, M$, let $\alpha_l=\lambda_l a_l$ and $\Delta_l=\lambda_l a_l\kappa_l$, where $\kappa_l=2\|\mathbf{B}_l\|_2+\eta_l$. Suppose that
\[
\kappa_l\leq K,\qquad a_l\eta_l\leq H,\qquad l=1,\ldots,M.
\]
Define the accumulated sensitivity envelope and squared
angular budget by
\[
S_M^{\mathrm{res}}=\sum_{l=1}^{M}\Delta_l,\qquad
A_M^{\mathrm{res}}=\sum_{l=1}^{M}|\theta_l|^2.
\]
Then the preceding estimates give
\[
S_M^{\mathrm{res}}\leq K\sum_{l=1}^{M}\lambda_l,
\qquad A_M^{\mathrm{res}}\leq H^2\sum_{l=1}^{M}\lambda_l^2.
\]
Since $M=L-1$, these estimates are equivalently
\[
\begin{array}{c|cc}
\lambda_{l} 
    &S_{L-1}^{\mathrm{res}} 
    &A_{L-1}^{\mathrm{res}}\\ \hline
1
    &\mathcal{O}(L)
    &\mathcal{O}(L) \\
{1}/{\sqrt{l}}
    &\mathcal{O}(\sqrt L)
    &\mathcal{O}(\log L) \\
{1}/{l}
    &\mathcal{O}(\log L)
    &\mathcal{O}(1).
\end{array}
\]
Thus, specializing the generic update count in Proposition~\ref{prop:depth-schedule-growth} changes only the endpoint of the sums and does not alter any of the stated asymptotic growth
rates.

\subsection{Full Jacobian and Layerwise Estimates}
\label{app:full-jacobian}
\paragraph{Indexing convention.}
Consistent with the retained-state indexing used in
Appendix~\ref{app:normalization-as-projection} and Appendix~\ref{app:fixed-radius-dynamics}, we consider the sequence $\mathbf{h}_1,\mathbf{h}_2,\ldots,\mathbf{h}_L$, where, for $k=1,\ldots,L-1$,
\[
\mathbf{h}_{k+1}=\operatorname{SpheretNorm}
\bigl(\mathbf{h}_k,\mathcal{F}_k(\mathbf{h}_k)\bigr).
\]
There are therefore $N=L-1$ individual SpheretNorm updates between $\mathbf{h}_1$ and $\mathbf{h}_L$.

For notational convenience in the Jacobian calculation, set $\mathbf{x}_k=\mathbf{h}_k$ for $k=1,\ldots,N+1$. Then
\[
\mathbf{x}_1=\mathbf{h}_1,\qquad\mathbf{x}_{N+1}=\mathbf{h}_L,
\]
and the $k$-th individual SpheretNorm update is written as
\[
\mathbf{x}_k\longmapsto\mathbf{x}_{k+1},\qquad k=1,\ldots,N.
\]
Proposition~\ref{prop:jacobian-bounds} uses $L$ as a generic count of individual SpheretNorm updates. To avoid conflict with the
retained-state count $L$ used here, we denote that generic
update count by $N$. For the retained-state sequence
$\mathbf{h}_1,\ldots, \mathbf{h}_L$, one has $N=L-1$.

\paragraph{Differential of a single SpheretNorm update.}
For each $k=1,\ldots, N$, let $\mathcal{F}_k$ denote the mapping associated with the $k$-th SpheretNorm update, and let $\alpha_k$ denote its step size. We assume that $\mathcal{F}_k$ is continuously differentiable in a neighborhood of $\mathbf{x}_k$. Throughout this subsection, all derivatives are taken with respect to $\mathbf{x}_k$,
while $p$ and $\alpha_k$ are treated as constants.

Define $r_k=\|\mathbf{x}_k\|_2$, $\mathbf{P}_k= \mathbf{I}- \frac{\mathbf{x}_k\mathbf{x}_k^\top}{r_k^2}$, $c_k=\frac{\mathbf{x}_k^\top\mathcal{F}_k(\mathbf{x}_k)}{r_k^2}$,
$\mathbf{z}_k=\mathcal{F}_k(\mathbf{x}_k)-c_k\mathbf{x}_k=
\mathbf{P}_k\mathcal{F}_k(\mathbf{x}_k)$, $s_k=\|\mathbf{z}_k\|_2$, and
\[
\eta_k=\frac{s_k}{r_k},
\quad
\beta_k=\alpha_k\eta_k,
\quad
\theta_k=p\arctan\left(\frac{\beta_k}{p}\right).
\]
Let $\mathbf{A}_k=D\mathcal{F}_k(\mathbf{x}_k)
=\frac{\partial\mathcal{F}_k(\mathbf{x}_k)}{\partial\mathbf{x}_k}$, $\mathbf{B}_k=\mathbf{A}_k-c_k\mathbf{I}$, and define $\mathbf{J}_k=\frac{\partial\mathbf{x}_{k+1}}
{\partial\mathbf{x}_k}$. For a perturbation
$\boldsymbol\xi\in\mathbb R^d$, differentiation of $\mathbf{z}_k=\mathcal{F}_k(\mathbf{x}_k)-c_k\mathbf{x}_k$ gives
\[
D\mathbf{z}_k[\boldsymbol\xi]=\mathbf{B}_k\boldsymbol\xi
-Dc_k[\boldsymbol\xi]\mathbf{x}_k.
\]
Since $\mathbf{P}_k\mathbf{x}_k=\boldsymbol{0}$, we immediately obtain $\mathbf{P}_kD\mathbf{z}_k[\boldsymbol\xi]=\mathbf{P}_k\mathbf{B}_k\boldsymbol\xi$. We next compute the derivative of $c_k$. Direct differentiation gives
$$
Dc_k[\boldsymbol\xi]=\frac{\boldsymbol\xi^\top\mathcal{F}_k(\mathbf{x}_k)+
\mathbf{x}_k^\top\mathbf{A}_k\boldsymbol\xi}{r_k^2}-2\frac{\mathbf{x}_k^\top\mathcal{F}_k(\mathbf{x}_k)}{r_k^4}\mathbf{x}_k^\top\boldsymbol\xi.
$$
Using $\mathcal{F}_k(\mathbf{x}_k)=\mathbf{z}_k+c_k\mathbf{x}_k$ and $\mathbf{A}_k=\mathbf{B}_k+c_k\mathbf{I}$, the terms containing
$c_k\mathbf{x}_k^\top\boldsymbol\xi$ cancel, yielding
\[
Dc_k[\boldsymbol\xi]=\frac{\mathbf{z}_k^\top\boldsymbol\xi+
\mathbf{x}_k^\top\mathbf{B}_k\boldsymbol\xi}{r_k^2}.
\]

Suppose first that $s_k>0$. Since $s_k=\|\mathbf{z}_k\|_2$ and $r_k=\|\mathbf{x}_k\|_2$, we have $\nabla_{\mathbf{x}_k}s_k=\frac{\mathbf{B}_k^\top\mathbf{z}_k}{s_k}$ and $\nabla_{\mathbf{x}_k}r_k
=\frac{\mathbf{x}_k}{r_k}$. Therefore,
\[
\nabla_{\mathbf{x}_k}\beta_k=
\alpha_k\left(\frac{\mathbf{B}_k^\top\mathbf{z}_k}{r_ks_k}-
\frac{s_k\mathbf{x}_k}{r_k^3}\right).
\]
Moreover, $\frac{d\theta_k}{d\beta_k}=\frac{p^2}{p^2+\beta_k^2}$, and hence
\[
\nabla_{\mathbf{x}_k}\theta_k=\frac{\alpha_kp^2}{p^2+\beta_k^2}
\left(\frac{\mathbf{B}_k^\top\mathbf{z}_k}{r_ks_k}
-\frac{s_k\mathbf{x}_k}{r_k^3}\right).
\]

For $s_k>0$, the $k$-th SpheretNorm update is $\mathbf{x}_{k+1}=\cos(\theta_k)\mathbf{x}_k+r_k\sin(\theta_k)\frac{\mathbf{z}_k}{s_k}$. Differentiating this expression and substituting the formula
for $Dc_k[\boldsymbol\xi]$ yields
\begin{align*}
\mathbf{J}_k={}&\cos(\theta_k)\mathbf{I}+\frac{\sin(\theta_k)}{r_ks_k}\left(
\mathbf{z}_k\mathbf{x}_k^\top-\mathbf{x}_k\mathbf{z}_k^\top\right)\\
&+\frac{r_k\sin(\theta_k)}{s_k}\left(\mathbf{P}_k-\frac{\mathbf{z}_k\mathbf{z}_k^\top}{s_k^2}
\right)\mathbf{B}_k\\
&+\left(\frac{r_k\cos(\theta_k)}{s_k}\mathbf{z}_k-\sin(\theta_k)\mathbf{x}_k\right)
\left(\nabla_{\mathbf{x}_k}\theta_k\right)^\top.
\end{align*}

Define $\mathbf Q_k=\cos(\theta_k)\mathbf{I}+\frac{\sin(\theta_k)}{r_ks_k}
\left(\mathbf{z}_k\mathbf{x}_k^\top-\mathbf{x}_k\mathbf{z}_k^\top\right)$. Then $\mathbf{J}_k=\mathbf Q_k+\mathbf E_k$, where
\begin{align*}
\mathbf E_k={}&\frac{r_k\sin(\theta_k)}{s_k}\left(\mathbf{P}_k-
\frac{\mathbf{z}_k\mathbf{z}_k^\top}{s_k^2}\right)\mathbf{B}_k\\
&+\left(\frac{r_k\cos(\theta_k)}{s_k}\mathbf{z}_k-
\sin(\theta_k)\mathbf{x}_k\right)\left(\nabla_{\mathbf{x}_k}\theta_k\right)^\top.
\end{align*}

\paragraph{Estimate of the perturbation term.}
Since $\mathbf{P}_k-\frac{\mathbf{z}_k\mathbf{z}_k^\top}{s_k^2}$ is the orthogonal projection onto the subspace of $T_{\mathbf{x}_k}\mathbb S_{r_k}^{d-1}$ orthogonal to $\mathbf{z}_k$, its operator norm is at most one.
Therefore,
\[
\left\|\frac{r_k\sin(\theta_k)}{s_k}\left(\mathbf{P}_k-
\frac{\mathbf{z}_k\mathbf{z}_k^\top}{s_k^2}\right)\mathbf{B}_k\right\|_2
\leq\frac{r_k|\sin(\theta_k)|}{s_k}\|\mathbf{B}_k\|_2.
\]
Since $|\beta_k|=|\alpha_k|\frac{s_k}{r_k}$, we obtain
\[
\left\|\frac{r_k\sin(\theta_k)}{s_k}\left(\mathbf{P}_k-
\frac{\mathbf{z}_k\mathbf{z}_k^\top}{s_k^2}\right)\mathbf{B}_k\right\|_2
\leq|\alpha_k|\frac{|\sin(\theta_k)|}{|\beta_k|}\|\mathbf{B}_k\|_2.
\]

For the second term in $\mathbf E_k$, the orthogonality
$\mathbf{x}_k^\top\mathbf{z}_k=0$ gives
\begin{align*}
\left\|\frac{r_k\cos(\theta_k)}{s_k}\mathbf{z}_k-
\sin(\theta_k)\mathbf{x}_k\right\|_2^2&=r_k^2\cos^2(\theta_k)+r_k^2\sin^2(\theta_k)=r_k^2.
\end{align*}
Hence,
\[
\left\|\frac{r_k\cos(\theta_k)}{s_k}\mathbf{z}_k-\sin(\theta_k)\mathbf{x}_k
\right\|_2=r_k.
\]
Furthermore,
\[
\left\|\nabla_{\mathbf{x}_k}\theta_k\right\|_2\leq
\frac{p^2}{p^2+\beta_k^2}\frac{|\alpha_k|\|\mathbf{B}_k\|_2+|\beta_k|}{r_k}.
\]
It follows that $\|\mathbf E_k\|_2\leq\delta_k^{\mathrm{exact}}$, where
\[
\delta_k^{\mathrm{exact}}=|\alpha_k|\frac{|\sin(\theta_k)|}{|\beta_k|}\|\mathbf{B}_k\|_2
+\frac{p^2}{p^2+\beta_k^2}\left(|\alpha_k|\|\mathbf{B}_k\|_2+|\beta_k|\right).
\]
The quotient $\frac{|\sin(\theta_k)|}{|\beta_k|}$ is understood through its continuous extension at $\beta_k=0$.

Since $|\sin(\theta_k)|\leq|\theta_k|\leq|\beta_k|$, $\frac{p^2}{p^2+\beta_k^2}\leq 1$, and $|\beta_k|=|\alpha_k|\eta_k$, we obtain
\[
\delta_k^{\mathrm{exact}}\leq|\alpha_k|\left(2\|\mathbf{B}_k\|_2+\eta_k\right).
\]
Define $\Delta_k=|\alpha_k|\left(2\|\mathbf{B}_k\|_2+\eta_k\right)$. Then
$\|\mathbf E_k\|_2\leq\Delta_k$.

\paragraph{The zero-tangent case.}
We now consider $s_k=0$. At the point under consideration, $\mathbf{z}_k=\boldsymbol{0}$, $\beta_k=0$. Direct differentiation of $\beta_k=\alpha_k\frac{\|\mathbf{z}_k\|_2}{r_k}$ is not valid at $\mathbf{z}_k=\boldsymbol{0}$, because the Euclidean norm is not differentiable at the origin. We therefore use the even-function representation established
in Appendix~\ref{app:p-angular-properties}.

Let $a_p(\beta)=\cos\left(p\arctan\left(\frac{\beta}{p}\right)\right)$ and
\[
b_p(\beta)=\begin{cases}
\displaystyle
\frac{\sin\left(p\arctan(\beta/p)\right)}{\beta},& \beta\neq0,\\[1.2ex]
1,& \beta=0.
\end{cases}
\]
There exist real-analytic functions
$\widehat A_p$ and $\widehat B_p$, defined near zero,
such that
\[
a_p(\beta)=\widehat A_p(\beta^2),\qquad
b_p(\beta)=\widehat B_p(\beta^2),
\]
with $\widehat A_p(0)=\widehat B_p(0)=1$. For $\mathbf{x}\neq\mathbf0$, define
\[
\mathbf{z}_k(\mathbf{x})=\left(\mathbf{I}-
\frac{\mathbf{x}\mathbf{x}^\top}{\|\mathbf{x}\|_2^2}\right)\mathcal{F}_k(\mathbf{x}),\quad
\omega_k(\mathbf{x})=\alpha_k^2
\frac{\|\mathbf{z}_k(\mathbf{x})\|_2^2}{\|\mathbf{x}\|_2^2}.
\]
The layer map can be written locally as
\[
\mathbf T_k(\mathbf{x})=
\widehat A_p\bigl(\omega_k(\mathbf{x})\bigr)\mathbf{x}+\alpha_k
\widehat B_p\bigl(\omega_k(\mathbf{x})\bigr)\mathbf{z}_k(\mathbf{x}).
\]
For a perturbation
$\boldsymbol\xi\in\mathbb R^d$,
$$
D\omega_k(\mathbf{x}_k)[\boldsymbol{\xi}]=\frac{2\alpha_k^2}{r_k^2}\mathbf{z}_k^\top
D\mathbf{z}_k[\boldsymbol{\xi}]-\frac{2\alpha_k^2s_k^2}{r_k^4}\mathbf{x}_k^\top\boldsymbol{\xi}.
$$
When $s_k=0$, we have $\omega_k(\mathbf{x}_k)=0$ and $D\omega_k(\mathbf{x}_k)[\boldsymbol{\xi}]=0$. Consequently,
\[
D\mathbf T_k(\mathbf{x}_k)[\boldsymbol\xi]=
\boldsymbol\xi+\alpha_kD\mathbf{z}_k[\boldsymbol\xi].
\]

Since $\mathbf{z}_k=\mathbf0$, one has $\mathcal{F}_k(\mathbf{x}_k)=c_k\mathbf{x}_k$. The derivative of $c_k$ therefore reduces to
\[
Dc_k[\boldsymbol\xi]=
\frac{\mathbf{x}_k^\top\mathbf{B}_k\boldsymbol\xi}{r_k^2}.
\]
Hence,
\begin{align*}
D\mathbf{z}_k[\boldsymbol\xi]&=\mathbf{B}_k\boldsymbol\xi
-Dc_k[\boldsymbol\xi]\mathbf{x}_k\\
&=\mathbf{B}_k\boldsymbol\xi-\frac{\mathbf{x}_k\mathbf{x}_k^\top}{r_k^2}\mathbf{B}_k\boldsymbol\xi
=\mathbf{P}_k\mathbf{B}_k\boldsymbol\xi.
\end{align*}
It follows that $\mathbf{J}_k=\mathbf{I}+\alpha_k\mathbf{P}_k\mathbf{B}_k$ when $s_k=0$. Define $\delta_k^{(0)}=|\alpha_k|\|\mathbf{P}_k\mathbf{B}_k\|_2$. Weyl's singular-value perturbation inequality gives
\[
\sigma_{\max}(\mathbf{J}_k)\leq 1+\delta_k^{(0)}
\]
and
\[
\sigma_{\min}(\mathbf{J}_k)\geq\max\{0,1-\delta_k^{(0)}\}.
\]
Moreover, $\delta_k^{(0)}\leq|\alpha_k|\|\mathbf{B}_k\|_2\leq\Delta_k$.

\paragraph{Singular-value bounds.}
For $s_k>0$, define $\mathbf{e}_{1,k}=\frac{\mathbf{x}_k}{r_k}$, $\mathbf{e}_{2,k}=\frac{\mathbf{z}_k}{s_k}$. On
$\operatorname{span}\{\mathbf{e}_{1,k},\mathbf{e}_{2,k}\}$,
the matrix of $\mathbf Q_k$, with respect to this
orthonormal basis, is
\[
\begin{pmatrix}
\cos(\theta_k) & -\sin(\theta_k)\\
\sin(\theta_k) & \cos(\theta_k)
\end{pmatrix}.
\]
Thus, $\mathbf Q_k$ acts as an isometry on this
two-dimensional subspace. On its orthogonal complement,
$\mathbf Q_k$ acts as multiplication by
$\cos(\theta_k)$. Consequently, we obtain that $\|\mathbf Q_k\|_2=1$ and $\sigma_{\min}(\mathbf Q_k)\geq|\cos(\theta_k)|$. For $s_k>0$, Weyl's inequality therefore gives
\[
\sigma_{\max}(\mathbf{J}_k)\leq 1+\Delta_k
\]
and
\[
\sigma_{\min}(\mathbf{J}_k)\geq\max\{0,|\cos(\theta_k)|-\Delta_k\}.
\]

Define
\[
q_k=
\begin{cases}
|\cos(\theta_k)|,& s_k>0,\\
1,& s_k=0,
\end{cases}
\]
and $m_k=\max\{0,q_k-\Delta_k\}$. Combining the cases $s_k>0$ and $s_k=0$, we obtain
\[
\sigma_{\max}(\mathbf{J}_k)\leq 1+\Delta_k,\qquad
\sigma_{\min}(\mathbf{J}_k)\geq m_k,\qquad k=1,\ldots,N.
\]

\paragraph{Depth-wise Jacobian.}
The end-to-end Jacobian over the retained residual stream is
\[
\mathbf{G}=\frac{\partial\mathbf{h}_L}{\partial\mathbf{h}_1}
=\frac{\partial\mathbf{x}_{N+1}}{\partial\mathbf{x}_1}.
\]
Since $N=L-1$, the chain rule gives $\mathbf{G}=\mathbf{J}_N\mathbf{J}_{N-1}\cdots\mathbf{J}_1
=\mathbf{J}_{L-1}\mathbf{J}_{L-2}\cdots\mathbf{J}_1$. By submultiplicativity,
\begin{align*}
\sigma_{\max}(\mathbf{G})&\leq\prod_{k=1}^{N}\sigma_{\max}(\mathbf{J}_k) \leq\prod_{k=1}^{N}(1+\Delta_k) \\
&\leq\exp\left(\sum_{k=1}^{N}\Delta_k\right).
\end{align*}
Equivalently, $\sigma_{\max}(\mathbf{G})\leq\exp\left(\sum_{k=1}^{L-1}\Delta_k\right)$. Similarly,
\[
\sigma_{\min}(\mathbf{G})\geq\prod_{k=1}^{N}\sigma_{\min}(\mathbf{J}_k)
\geq\prod_{k=1}^{N}m_k=\prod_{k=1}^{L-1}m_k.
\]

\paragraph{A consequence of exact norm preservation.}
Since $\|\mathbf{x}_{k+1}\|_2^2=
\|\mathbf{x}_k\|_2^2$ for every admissible $\mathbf{x}_k$, differentiation with
respect to $\mathbf{x}_k$ yields $\mathbf{J}_k^\top\mathbf{x}_{k+1}=\mathbf{x}_k$. Therefore,
\[
\|\mathbf{J}_k\|_2=\|\mathbf{J}_k^\top\|_2\geq
\frac{\|\mathbf{J}_k^\top\mathbf{x}_{k+1}\|_2}{\|\mathbf{x}_{k+1}\|_2}
=\frac{\|\mathbf{x}_k\|_2}{\|\mathbf{x}_{k+1}\|_2}=1.
\]
Thus, an individual SpheretNorm update Jacobian cannot be
a strict contraction in every ambient direction, although it
may still strongly contract particular tangential directions.

\subsection{Angle Control Alone Is Insufficient}
\label{app:angle-control-insufficient}
The intrinsic rotation angle controls the geometric component
of the update Jacobian, but it does not control the differential sensitivity of the underlying mapping. The following example shows that even zero angular displacement provides no uniform upper bound on the update-Jacobian norm and does not guarantee nonsingularity.

\begin{proposition}
For every $R>1$ and every $\alpha>0$, there exist a
linear mapping
$\mathcal{F}:\mathbb{R}^2\to\mathbb{R}^2$ and a hidden state
$\mathbf{h}\in\mathbb{S}^1$ such that the corresponding
SpheretNorm update satisfies $\beta=0$ and $\theta^{(p)}=0$, while its update Jacobian satisfies $\|\mathbf{J}\|_2=R$.
\end{proposition}

\begin{proof}
Let $\mathbf{h}=\mathbf{e}_1=\begin{pmatrix}
1\\
0
\end{pmatrix}
$ and define $\mathcal{F}(\mathbf{x})=\mathbf{A}\mathbf{x}$ and $\mathbf{A}=\begin{pmatrix}
\mu & 0\\
0 & \mu+K
\end{pmatrix}$, where $\mu,K\in\mathbb{R}$. When $\mathbf{h}=\mathbf{e}_1$, we have $\mathcal{F}(\mathbf{h})=\mu\mathbf{e}_1$ and $c=\mathbf{h}^\top\mathcal{F}(\mathbf{h})=\mu$. Therefore,
\[
\mathbf{B}=\mathbf{A}-c\mathbf{I}=
\begin{pmatrix}
0 & 0\\
0 & K
\end{pmatrix}.
\]

Moreover, $\mathbf{P}=\mathbf{I}-\mathbf{h}\mathbf{h}^\top
=\begin{pmatrix}
0&0\\
0&1
\end{pmatrix}$, and hence $\mathbf{z}=\mathbf{P}\mathcal{F}(\mathbf{h})
=\boldsymbol{0}$. Therefore, $\beta=0$ and $\theta^{(p)}=0$ for every $p>0$. Since $\mathbf{z}=\boldsymbol{0}$, the zero-tangent formula established in Appendix~\ref{app:full-jacobian} gives
\[
\mathbf{J}=\mathbf{I}+\alpha\mathbf{P}\mathbf{B}
=\begin{pmatrix}
1&0\\
0&1+\alpha K
\end{pmatrix}.
\]
Choosing $K=\frac{R-1}{\alpha}$ gives
$\mathbf{J}=
\begin{pmatrix}
1&0\\
0&R
\end{pmatrix}$, and therefore $\|\mathbf{J}\|_2=R$.
\end{proof}

The same construction also shows that zero angular displacement
does not imply nonsingularity. Indeed, choosing $K=-\frac{1}{\alpha}$ gives $\mathbf{J}=
\begin{pmatrix}
1&0\\
0&0
\end{pmatrix}$, so that $\sigma_{\min}(\mathbf{J})=0$. Therefore, angular control must be supplemented by control of the update-sensitivity term $|\alpha|\|\mathbf{B}\|_2$.

\subsection{Dyadic Angular Budget of the Depth Schedules}
\label{app:dyadic-angular-budget}

The global estimates in Proposition~\ref{prop:depth-schedule-growth} compare the accumulated stability envelopes over the entire sequence of SpheretNorm updates. We now examine how the available angular motion is distributed across different depth scales.

Under the retained-state indexing $\mathbf h_1,\mathbf h_2,\ldots,\mathbf h_L$, the individual SpheretNorm updates are indexed by $l=1,\ldots, L-1$. For an integer $m\geq 1$ satisfying $2m\leq L$, define the dyadic update interval $I_m=\{m,m+1,\ldots,2m-1\}$. The condition $2m\leq L$ ensures that $I_m\subseteq\{1,\ldots,L-1\}$.

\begin{proposition}\label{prop:dyadic-angular-budget}
Fix an optimization iteration $\tau$ and let $p>0$.
Suppose that there exist constants
\[
0<H_-\leq H_+<\infty
\]
and $\varepsilon>0$ such that, for every $l\in I_m$,
\[
H_-\leq a_l^{(\tau)}\eta_l^{(\tau)}\leq H_+
\quad\mbox{ and }\quad
\left|\beta_l^{(\tau)}\right|\leq\varepsilon.
\]
Then there exist constants $C_-,C_+>0$, independent of
$m$ and $L$, such that
\[
C_-\sum_{l\in I_m}\lambda_l^2\leq\sum_{l\in I_m}\left|\theta_{l,\tau}^{(p)}\right|^2\leq C_+
\sum_{l\in I_m}\lambda_l^2.
\]
Consequently, $\sum_{l\in I_m}\left|\theta_{l,\tau}^{(p)}\right|^2=\Theta(1)$ for the layerwise square-root schedule $\lambda_l=l^{-1/2}$, while
$\sum_{l\in I_m}\left|\theta_{l,\tau}^{(p)}\right|^2=\Theta(\frac{1}{m})$ for the harmonic schedule $\lambda_l=l^{-1}$.
\end{proposition}

\begin{proof}
Define
\[
\phi_p(t)=\begin{cases}
\displaystyle
\frac{p\arctan(t/p)}{t},&t>0,\\[3mm]
1,&t=0.
\end{cases}
\]
The function $\phi_p$ is continuous and strictly positive on
the compact interval $[0,\varepsilon]$. Hence there exists
a constant
$c_{p,\varepsilon}=\min_{0\leq t\leq\varepsilon}\phi_p(t)>0$ such that
\[
c_{p,\varepsilon}t\leq p\arctan\left(\frac{t}{p}\right)
\leq t,\qquad 0\leq t\leq\varepsilon.
\]
Applying this estimate with
$t=|\beta_l^{(\tau)}|$ gives $c_{p,\varepsilon}\left|\beta_l^{(\tau)}\right|\leq
\left|\theta_{l,\tau}^{(p)}\right|\leq
\left|\beta_l^{(\tau)}\right|$. Since $\left|\beta_l^{(\tau)}\right|=\lambda_l a_l^{(\tau)}\eta_l^{(\tau)}$, the assumptions imply $c_{p,\varepsilon}H_-\lambda_l\leq\left|\theta_{l,\tau}^{(p)}
\right|\leq H_+\lambda_l$. Squaring and summing over $I_m$ yields
\[
c_{p,\varepsilon}^2H_-^2\sum_{l\in I_m}\lambda_l^2\leq
\sum_{l\in I_m}\left|\theta_{l,\tau}^{(p)}\right|^2
\leq H_+^2\sum_{l\in I_m}\lambda_l^2.
\]
Thus, the first assertion holds with $C_-=c_{p,\varepsilon}^2H_-^2$ and $C_+=H_+^2$. For the layerwise square-root schedule
$\lambda_l=l^{-1/2}$, we have
\[
\sum_{l=m}^{2m-1}\lambda_l^2=\sum_{l=m}^{2m-1}\frac{1}{l}.
\]
Since $m\leq l\leq 2m-1$ for every $l\in I_m$, we have $\frac{1}{2m-1}\leq\frac{1}{l}\leq\frac{1}{m}$. Since $I_m$ contains exactly $m$ indices, it follows that $\frac{m}{2m-1}\leq\sum_{l=m}^{2m-1}\frac{1}{l}\leq 1$. Hence,
\[
\sum_{l=m}^{2m-1}\lambda_l^2=\Theta(1).
\]

For the harmonic schedule $\lambda_l=l^{-1}$, we have
$\sum_{l=m}^{2m-1}\lambda_l^2=\sum_{l=m}^{2m-1}\frac{1}{l^2}$. Using again $m\leq l\leq 2m-1$, we obtain
$\frac{1}{(2m-1)^2}\leq\frac{1}{l^2}\leq\frac{1}{m^2}$. Therefore,
\[
\frac{m}{(2m-1)^2}\leq\sum_{l=m}^{2m-1}\frac{1}{l^2}\leq\frac{1}{m}.
\]
Since $\frac{m}{(2m-1)^2}\geq\frac{1}{4m}$, we conclude that
\[
\frac{1}{4m}\leq\sum_{l=m}^{2m-1}\lambda_l^2\leq\frac{1}{m}.
\]
Hence,
\[
\sum_{l=m}^{2m-1}\lambda_l^2=\Theta\left(\frac{1}{m}\right).
\]
The stated conclusions follow.
\end{proof}

Proposition~\ref{prop:dyadic-angular-budget} distinguishes the distribution of squared angular motion from the global worst-case stability envelope. The harmonic schedule gives the smaller global
accumulated-sensitivity bound, but allocates progressively
less squared angular motion to deeper dyadic update intervals. By contrast, the layerwise square-root schedule retains an approximately constant squared angular budget across logarithmic scales of update depth, allowing the transformations associated with deeper SpheretNorm updates to remain nontrivial.

\end{document}